\documentclass{article}

\usepackage{amsmath,bm}
\usepackage{amssymb}
\usepackage[margin=1in]{geometry}
\usepackage{amsthm}
\usepackage{algorithmic}
\usepackage{url}
\usepackage[ampersand]{easylist}
\usepackage{graphicx}
\usepackage{caption}
\usepackage{subcaption}
\usepackage{tikz}
\usepackage{color}
\usepackage{multirow}
\usepackage[ruled]{algorithm2e}
\usetikzlibrary{fit,positioning}

\newtheorem{lemma}{Lemma}

\newcommand*\samethanks[1][\value{footnote}]{\footnotemark[#1]}

\ListProperties(Hide=100, Hang=true, Progressive=3ex, Style*=-- ,
Style2*=$\bullet$ ,Style3*=$\circ$ ,Style4*=\tiny$\blacksquare$ )
\newcommand{\myGlobalTransformation}[2]
{
	\pgftransformcm{1}{0}{0}{1}{\pgfpoint{#1cm}{#2cm}}
}

\title{Online Prediction of Dyadic Data with Heterogeneous Matrix Factorization}

\author{Guangyong Chen \thanks{All authors are with the Chinese University of Hong Kong, Shatin, N.T., Hong Kong. The contact information can be found on the website http://www.cse.cuhk.edu.hk/$\sim$pheng/}
	\and Fengyuan Zhu \samethanks
	\and Pheng Ann Heng  \samethanks }

\begin{document}

\maketitle

\begin{abstract}
Dyadic Data Prediction (DDP) is an important problem in many research areas. This paper develops a novel fully Bayesian nonparametric framework which integrates two popular and complementary approaches, discrete mixed membership modeling and continuous latent factor modeling into a unified Heterogeneous Matrix Factorization~(HeMF) model, which can predict the unobserved dyadics accurately. The HeMF can determine the number of communities automatically and exploit the latent linear structure for each bicluster efficiently. We propose a Variational Bayesian method to estimate the parameters and missing data. We further develop a novel online learning approach for Variational inference and use it for the online learning of HeMF, which can efficiently cope with the important large-scale DDP problem. We evaluate the performance of our method on the EachMoive, MovieLens and Netflix Prize collaborative filtering datasets. The experiment shows that, our model outperforms state-of-the-art methods on all benchmarks. Compared with Stochastic Gradient Method (SGD), our online learning approach achieves significant improvement on the estimation accuracy and robustness.
\end{abstract}

\section{Introduction}\label{sec:introduction}

Over the past few decades, dyadic data prediction (DDP) problem \cite{1999HofmannDDP} has attracted lots of research attentions from many areas, including pattern recognition, neural computation, statistics, machine learning and data mining. Dyadic means an ordered pair of objects and the goal of DDP is to predict the value of unseen dyadic given the values of observed ones. Mathematically, Candes and Tao \cite{2010CandesMC} defined the DDP problem as the process of 
\emph{ "recovering the whole matrix $R\in\mathbb{R}^{U\times M}$ from only a sampled set of its entries $\{r_{ij}\}_{(i,j)\in\Omega}$, with  $\Omega$ being a subset of complete  set of entries $[U]\times[M]$"}. Here and in the sequel, $[U]$ denotes the set $\{1,\ldots,U\}$. As illustrated in Fig. \ref{fig_ProDef}, the task of DDP is to estimate the values of white entries~(unobserved dyadic) in each matrix given the colored entries denoting the observed dyadic. An important application of this method is the recommendation system which predicts the preference of users to a special item based on the learned users' taste information and item's latent properties. This problem also finds large amount of practical applications in other research fields, including the image completion task in the filed of computer vision, blind affinity prediction in Bioinformatics \cite{2014NatarGene}, click prediction in Web search \cite{2013SuiLink}, and other applications \cite{2006PascualnsNMF,2011SandlenMF,2012LiucnMF,2015ZitnikPR,2014GillisFMF}.


The DDP is an ill-posed problem and it is impossible to determine the missing entries without making any assumption about the matrix $R$. A popular approach for this problem is to assume that the unknown matrix $R$ has low rank $L$ or has approximately low rank $L$. It has been proved that under certain incoherence assumptions on the singular vectors of the matrix $R$, exact prediction is possible by solving a convenient convex program as long as the number of samples is on the order of $NL\log(N)$ with $N=\max\{U,M\}$ \cite{2009CandesMC,2010CandesMC,2010CandesNMC}. These seminal results have boomed the development of this filed. However, in some practical situations, DDP still remains intractable when it is impossible to acquire enough entries to satisfy the minimum number condition. This problem is common in the area of recommendation system. For example, in the dataset of \emph{Movielens 100k}, the rating matrix $R$ is highly sparse, where only around $5\%$ entries are observed. To tackle this issue, \cite{1999HofmannBPMF,2004MarlinBPMF,2008SalakhutdinovBPMF} proposed the Bayesian Probabilistic Matrix Factorization (BPMF) methods which model $R$ as the product of an user coefficient matrix $A\in\mathbb{R}^{L\times U}$ and an item coefficient matrix $B\in\mathbb{R}^{L\times M}$, with both matrices $A$ and $B$ following Gaussian distributions. 

Though BPMF methods perform remarkably well on the DDP task, they fail to capture the heterogeneous nature of objects and their interactions because of the Gaussian assumption. The heterogeneous natural of objects is common in real commercial recommendation systems. For example, in \emph{Douban.com}, users are usually clustered into different communities depending on their own interests, and items are often categorized into multiple groups based on their own properties. Thus, it is not surprising that users from different groups may have quite distinct opinions for some movies, such as \emph{"Napoleon Dynamite", "Kill Bill: Volume 1", "Sideways"}. The Gaussian assumption of BPMF cannot capture this feature properly resulting to poor prediction results. This issue is significant and should be well considered when developing algorithms for solving DDP tasks, or the prediction performance will be highly affected empirically. To tackle this issue, models like Mixed Membership Stochastic Blockmodel~\cite{2009AiroldiMMSD} and Bi-LDA~\cite{porteous2008multi} have been proposed with the assumption that objects are generated from different communities. They introduce context dependence by allowing each object to select a new topic for each new interaction. However, the relatively poor predictive performance of Bi-LDA suggests that the blockmodel assumption is still too restrictive. This paper proposes a novel Heterogeneous Matrix Factorization (HeMF) model to unify discrete mixed membership model with BPMF, where the missing entries are estimated from the divided homogeneous sub-matrices. Fig. \ref{fig_ProDef} (b) illustrates the HeMF for better understanding. The proposed HeMF model can be also viewed as a new type of bicluster model, where each bicluster have an intrinsic linear structure. Compared with the traditional Matrix Factorization approaches \cite{2010MackeyM3F, 2014KimSideInf,2013ParkSideInf,2010AdamsSideInf,2010PorteousSideInf}, the proposed model HeMF incorporates the community membership information from a new perspective. Because the number of bi-clusters are usually unknown, this work further incorporates the Bayesian nonparametric technique and let the data determine the model complexity automatically. This paper further derives an efficient batch-inference algorithm for HeMF under the principle of Variational Bayesian (VB) \cite{2000AttiasVB}. As demonstrated empirically, the derived method converges much faster than traditional sampling methods, and gives better prediction performance.

\begin{figure*}[htp!]
	\centering
	\begin{tikzpicture}[scale = 0.2]
	
	\tikzstyle{timeNode}=[circle, minimum size = 3mm, thick, fill =black!60!green, node distance = 7mm];
	
	\node[timeNode] (t1) at (8,-6.5){};
	\node[timeNode] (t2) at (46,-6.5){};
	\node[timeNode] (t3) at (66,-6.5){};
	
	\draw[color =black!60!green,line width=2,->,>=stealth] (0,-6.5) --(74,-6.5);
	\begin{scope}
	\myGlobalTransformation{0}{0};
	\draw [black!50,step=2] 
	grid (16,16);
	\filldraw[fill=red] (0,2) rectangle (2,4);
	\filldraw[fill=red] (2,6) rectangle (4,8);
	\filldraw[fill=red] (4,10) rectangle (6,12);
	\filldraw[fill=red] (4,0) rectangle (6,2);
	\filldraw[fill=red] (6,14) rectangle (8,16);
	\filldraw[fill=red] (8,8) rectangle (10,10);
	\filldraw[fill=red] (10,2) rectangle (12,4);
	\filldraw[fill=red] (12,12) rectangle (14,14);
	\filldraw[fill=red] (14,4) rectangle (16,6);
	
	\node[rotate=30,scale=0.6] at (1,16.7) {\textcolor[rgb]{0,1,1} {I1}};
	\node[rotate=30,scale=0.6] at (3,16.7) {\textcolor[rgb]{0,1,1} {I2} };
	\node[rotate=30,scale=0.6] at (5,16.7) {\textcolor[rgb]{1,0,1}{I3}};
	\node[rotate=30,scale=0.6] at (7,16.7) {\textcolor[rgb]{0,1,1}{I4}};
	\node[rotate=30,scale=0.6] at (9,16.7) {\textcolor[rgb]{1,0,1}{I5}};
	\node[rotate=30,scale=0.6] at (11,16.7) {\textcolor[rgb]{1,0,1}{I6}};
	\node[rotate=30,scale=0.6] at (13,16.7) {\textcolor[rgb]{0,1,1}{I7}};
	\node[rotate=30,scale=0.6] at (15,16.7) {\textcolor[rgb]{1,0,1}{I8}};
	
	\node[rotate=60,scale=0.6] at (-0.7,15) {\textcolor[rgb]{0,0,0.5}{U1}};
	\node[rotate=60,scale=0.6] at (-0.7,13) {\textcolor[rgb]{1,0.27,0}{U2}};
	\node[rotate=60,scale=0.6] at (-0.7,11) {\textcolor[rgb]{0,0,0.5}{U3}};
	\node[rotate=60,scale=0.6] at (-0.7,9) {\textcolor[rgb]{0,0,0.5}{U4}};
	\node[rotate=60,scale=0.6] at (-0.7,7) {\textcolor[rgb]{1,0.27,0}{U5}};
	\node[rotate=60,scale=0.6] at (-0.7,5) {\textcolor[rgb]{1,0.27,0}{U6}};
	\node[rotate=60,scale=0.6] at (-0.7,3) {\textcolor[rgb]{0,0,0.5}{U7}};
	\node[rotate=60,scale=0.6] at (-0.7,1) {\textcolor[rgb]{1,0.27,0}{U8}};
	
	\node[scale=0.9] at (8,-5.5) {Time $t$};
	\node[scale=0.9] at (8,-3.7) {$(a)$};
	\end{scope}
	
	\begin{scope}
	\myGlobalTransformation{18}{9};
	\draw [black!50,step=2] 
	grid (8,8);
	\filldraw[fill=red] (2,4) rectangle (4,6);
	\filldraw[fill=red] (6,6) rectangle (8,8);
	
	\node[rotate=30,scale=0.6] at (1,8.7) {\textcolor[rgb]{0,1,1} {I1}};
	\node[rotate=30,scale=0.6] at (3,8.7) {\textcolor[rgb]{0,1,1} {I2} };
	\node[rotate=30,scale=0.6] at (5,8.7) {\textcolor[rgb]{0,1,1}{I4}};
	\node[rotate=30,scale=0.6] at (7,8.7) {\textcolor[rgb]{0,1,1}{I7}};
	
	\node[rotate=60,scale=0.6] at (-0.7,7) {\textcolor[rgb]{1,0.27,0}{U2}};
	\node[rotate=60,scale=0.6] at (-0.7,5) {\textcolor[rgb]{1,0.27,0}{U5}};
	\node[rotate=60,scale=0.6] at (-0.7,3) {\textcolor[rgb]{1,0.27,0}{U6}};
	\node[rotate=60,scale=0.6] at (-0.7,1) {\textcolor[rgb]{1,0.27,0}{U8}};
	
	\node[scale=0.9] at (4,10) {Submatrix $1$};
	\end{scope}
	
	\begin{scope}
	\myGlobalTransformation{18}{-1};
	\draw [black!50,step=2] 
	grid (8,8);
	\filldraw[fill=red] (4,0) rectangle (6,2);
	\filldraw[fill=red] (2,2) rectangle (4,4);
	\filldraw[fill=red] (0,4) rectangle (2,6);
	
	\node[rotate=30,scale=0.6] at (1,8.7) {\textcolor[rgb]{1,0,1} {I3}};
	\node[rotate=30,scale=0.6] at (3,8.7) {\textcolor[rgb]{1,0,1} {I5} };
	\node[rotate=30,scale=0.6] at (5,8.7) {\textcolor[rgb]{1,0,1}{I6}};
	\node[rotate=30,scale=0.6] at (7,8.7) {\textcolor[rgb]{1,0,1}{I8}};
	
	\node[rotate=60,scale=0.6] at (-0.7,7) {\textcolor[rgb]{0,0,0.5}{U1}};
	\node[rotate=60,scale=0.6] at (-0.7,5) {\textcolor[rgb]{0,0,0.5}{U3}};
	\node[rotate=60,scale=0.6] at (-0.7,3) {\textcolor[rgb]{0,0,0.5}{U4}};
	\node[rotate=60,scale=0.6] at (-0.7,1) {\textcolor[rgb]{0,0,0.5}{U7}};
	\node[scale=0.9] at (4,-1.5) {Submatrix $3$};
	\end{scope}

	\begin{scope}
	\myGlobalTransformation{28}{9};
	\draw [black!50,step=2] 
	grid (8,8);
	\filldraw[fill=red] (0,0) rectangle (2,2);
	\filldraw[fill=red] (6,2) rectangle (8,4);
	
	\node[rotate=30,scale=0.6] at (1,8.7) {\textcolor[rgb]{1,0,1} {I3}};
	\node[rotate=30,scale=0.6] at (3,8.7) {\textcolor[rgb]{1,0,1} {I5} };
	\node[rotate=30,scale=0.6] at (5,8.7) {\textcolor[rgb]{1,0,1}{I6}};
	\node[rotate=30,scale=0.6] at (7,8.7) {\textcolor[rgb]{1,0,1}{I8}};
	
	\node[rotate=60,scale=0.6] at (-0.7,7) {\textcolor[rgb]{1,0.27,0}{U2}};
	\node[rotate=60,scale=0.6] at (-0.7,5) {\textcolor[rgb]{1,0.27,0}{U5}};
	\node[rotate=60,scale=0.6] at (-0.7,3) {\textcolor[rgb]{1,0.27,0}{U6}};
	\node[rotate=60,scale=0.6] at (-0.7,1) {\textcolor[rgb]{1,0.27,0}{U8}};
	\node[scale=0.9] at (4,10) {Submatrix $2$};
	\end{scope}

	\begin{scope}
	\myGlobalTransformation{28}{-1};
	\draw [black!50,step=2] 
	grid (8,8);
	\filldraw[fill=red] (0,0) rectangle (2,2);
	\filldraw[fill=red] (4,6) rectangle (6,8);
	
	\node[rotate=30,scale=0.6] at (1,8.7) {\textcolor[rgb]{0,1,1} {I1}};
	\node[rotate=30,scale=0.6] at (3,8.7) {\textcolor[rgb]{0,1,1} {I2} };
	\node[rotate=30,scale=0.6] at (5,8.7) {\textcolor[rgb]{0,1,1}{I4}};
	\node[rotate=30,scale=0.6] at (7,8.7) {\textcolor[rgb]{0,1,1}{I7}};
	
	\node[rotate=60,scale=0.6] at (-0.7,7) {\textcolor[rgb]{0,0,0.5}{U1}};
	\node[rotate=60,scale=0.6] at (-0.7,5) {\textcolor[rgb]{0,0,0.5}{U3}};
	\node[rotate=60,scale=0.6] at (-0.7,3) {\textcolor[rgb]{0,0,0.5}{U4}};
	\node[rotate=60,scale=0.6] at (-0.7,1) {\textcolor[rgb]{0,0,0.5}{U7}};
	\node[scale=0.9] at (4,-1.5) {Submatrix $4$};
	\node[scale=0.9] at (-0.5,-3) {$(b)$};
	\end{scope}
	
	\begin{scope}
	\myGlobalTransformation{38}{0};
	\draw [black!50,step=2] 
	grid (16,16);
	\filldraw[fill=red] (0,2) rectangle (2,4);
	\filldraw[fill=red] (2,6) rectangle (4,8);
	\filldraw[fill=red] (4,10) rectangle (6,12);
	\filldraw[fill=red] (4,0) rectangle (6,2);
	\filldraw[fill=red] (6,14) rectangle (8,16);
	\filldraw[fill=red] (8,8) rectangle (10,10);
	\filldraw[fill=red] (10,2) rectangle (12,4);
	\filldraw[fill=blue] (12,12) rectangle (14,14);
	\filldraw[fill=red] (14,4) rectangle (16,6);
	
	\filldraw[fill=blue] (8,4) rectangle (10,6);
	\filldraw[fill=blue] (0,14) rectangle (2,16);
	\filldraw[fill=blue] (12,8) rectangle (14,10);
	\filldraw[fill=blue] (14,0) rectangle (16,2);
	
	\node[rotate=30,scale=0.6] at (1,16.7) {\textcolor[rgb]{0,1,1}{I1}};
	\node[rotate=30,scale=0.6] at (3,16.7) {\textcolor[rgb]{0,1,1}{I2}};
	\node[rotate=30,scale=0.6] at (5,16.7) {\textcolor[rgb]{1,0,1}{I3}};
	\node[rotate=30,scale=0.6] at (7,16.7) {\textcolor[rgb]{0,1,1}{I4}};
	\node[rotate=30,scale=0.6] at (9,16.7) {\textcolor[rgb]{1,0,1}{I5}};
	\node[rotate=30,scale=0.6] at (11,16.7) {\textcolor[rgb]{1,0,1}{I6}};
	\node[rotate=30,scale=0.6] at (13,16.7) {\textcolor[rgb]{0,1,1}{I7}};
	\node[rotate=30,scale=0.6] at (15,16.7) {\textcolor[rgb]{1,0,1}{I8}};
	
	\node[rotate=60,scale=0.6] at (-0.7,15) {\textcolor[rgb]{0,0,0.5}{U1}};
	\node[rotate=60,scale=0.6] at (-0.7,13) {\textcolor[rgb]{1,0.27,0}{U2}};
	\node[rotate=60,scale=0.6] at (-0.7,11) {\textcolor[rgb]{0,0,0.5}{U3}};
	\node[rotate=60,scale=0.6] at (-0.7,9) {\textcolor[rgb]{0,0,0.5}{U4}};
	\node[rotate=60,scale=0.6] at (-0.7,7) {\textcolor[rgb]{1,0.27,0}{U5}};
	\node[rotate=60,scale=0.6] at (-0.7,5) {\textcolor[rgb]{1,0.27,0}{U6}};
	\node[rotate=60,scale=0.6] at (-0.7,3) {\textcolor[rgb]{0,0,0.5}{U7}};
	\node[rotate=60,scale=0.6] at (-0.7,1) {\textcolor[rgb]{1,0.27,0}{U8}};

	\node[scale=0.9] at (8,-5.5) {Time $t+1$};
	\node[scale=0.9] at (8,-3.7) {$(c)$};
	\end{scope}
	
	\begin{scope}
	\myGlobalTransformation{56}{-1};
	\draw [black!50,step=2] 
	grid (18,18);
	\filldraw[fill=red] (0,4) rectangle (2,6);
	\filldraw[fill=red] (2,8) rectangle (4,10);
	\filldraw[fill=red] (4,12) rectangle (6,14);
	\filldraw[fill=red] (4,2) rectangle (6,4);
	\filldraw[fill=red] (6,16) rectangle (8,18);
	\filldraw[fill=red] (8,10) rectangle (10,12);
	\filldraw[fill=red] (10,4) rectangle (12,6);
	\filldraw[fill=red] (12,14) rectangle (14,16);
	\filldraw[fill=red] (14,6) rectangle (16,8);
	
	\filldraw[fill=blue] (8,6) rectangle (10,8);
	\filldraw[fill=blue] (0,16) rectangle (2,18);
	\filldraw[fill=blue] (12,10) rectangle (14,12);
	\filldraw[fill=blue] (14,2) rectangle (16,4);
	
	\filldraw[fill=green] (2,14) rectangle (4,16);
	\filldraw[fill=green] (8,0) rectangle (10,2);
	\filldraw[fill=green] (0,0) rectangle (2,2);
	\filldraw[fill=green] (16,12) rectangle (18,14);
	
	\node[rotate=30,scale=0.6] at (1,18.7) {\textcolor[rgb]{0,1,1}{I1}};
	\node[rotate=30,scale=0.6] at (3,18.7) {\textcolor[rgb]{0,1,1}{I2}};
	\node[rotate=30,scale=0.6] at (5,18.7) {\textcolor[rgb]{1,0,1}{I3}};
	\node[rotate=30,scale=0.6] at (7,18.7) {\textcolor[rgb]{0,1,1}{I4}};
	\node[rotate=30,scale=0.6] at (9,18.7) {\textcolor[rgb]{1,0,1}{I5}};
	\node[rotate=30,scale=0.6] at (11,18.7) {\textcolor[rgb]{1,0,1}{I6}};
	\node[rotate=30,scale=0.6] at (13,18.7) {\textcolor[rgb]{0,1,1}{I7}};
	\node[rotate=30,scale=0.6] at (15,18.7) {\textcolor[rgb]{1,0,1}{I8}};
	\node[rotate=30,scale=0.6] at (17,18.7) {I9};
	
	\node[rotate=60,scale=0.6] at (-0.7,17) {\textcolor[rgb]{0,0,0.5}{U1}};
	\node[rotate=60,scale=0.6] at (-0.7,15) {\textcolor[rgb]{1,0.27,0}{U2}};
	\node[rotate=60,scale=0.6] at (-0.7,13) {\textcolor[rgb]{0,0,0.5}{U3}};
	\node[rotate=60,scale=0.6] at (-0.7,11) {\textcolor[rgb]{0,0,0.5}{U4}};
	\node[rotate=60,scale=0.6] at (-0.7,9) {\textcolor[rgb]{1,0.27,0}{U5}};
	\node[rotate=60,scale=0.6] at (-0.7,7) {\textcolor[rgb]{1,0.27,0}{U6}};
	\node[rotate=60,scale=0.6] at (-0.7,5) {\textcolor[rgb]{0,0,0.5}{U7}};
	\node[rotate=60,scale=0.6] at (-0.7,3) {\textcolor[rgb]{1,0.27,0}{U8}};
	\node[rotate=60,scale=0.6] at (-0.7,1) {U9};
	
	\node[scale=0.9] at (8,-4.5) {Time $t+2$};
	\node[scale=0.9] at (8,-3) {$(d)$};
	\end{scope} 
	
	\end{tikzpicture}
	\caption{An illustration for the rating matrix $R$ evolved with time, where $I1$ means the $1^{st}$ item and $U1$ represents the $1^{st}$ user. For objects from different communities, their index are marked by different colors accordingly. (a) denotes the sparse rating matrix acquired at time $t$. (b) contains $4$ homogeneous sub-matrices divided by the rating matrix observed at time $t$. (c) denotes the rating matrix observed at time $t+1$, where new observations are marked by the blue color. Note that the $2^{nd}$ user changed his rating score for the $7^{st}$ item. (d) contains the rating matrix observed at time $t+2$, where new observations are marked by the color green, and new item and user are registered.}
	\label{fig_ProDef}
\end{figure*}
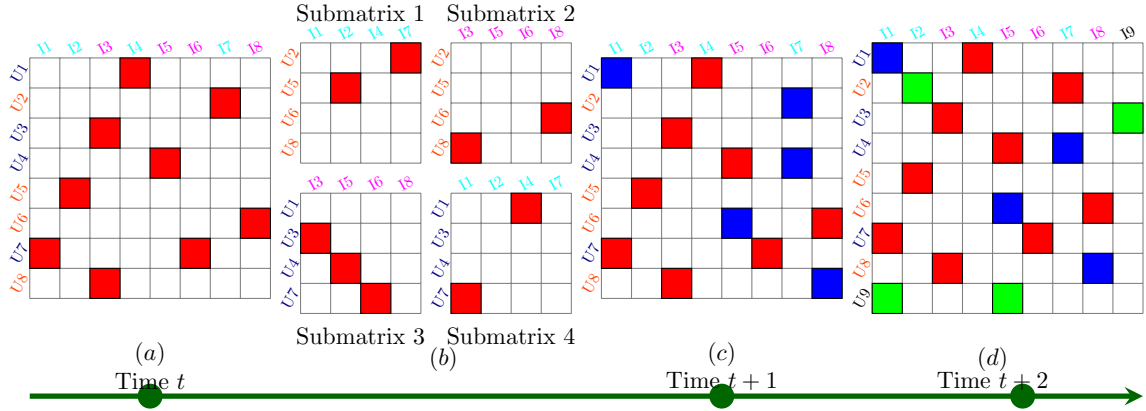

The era of \emph{Big Data} presents new challenges for our DDP task. Many real world applications involve massive amount of data that even cannot be accommodated entirely in the memory. Moreover, the amount of data even increases without any limitation as time goes on. For example, users give new ratings to items at nearly all time in Amazon, and there are always new user accounts registered online and new items launched into the market. Mathematically speaking, new entries are labeled in the rating matrix $R$ continuously, and the size of $R$ will become larger as time goes on. This process can be better understood in Fig. \ref{fig_ProDef}. Because both sampling and variational methods maintain the entire configuration and perform iterative updates of multiple passes, these inference strategies limit their implementations on the massive stream data. Stochastic Gradient Descent (SGD) and sampling approaches have been applied in a sequential setting for matrix factorization \cite{2007MnihPMF,2015JayaOMF}. However, they are not embarrassingly online and hence cannot be directly applied to the stream data. As observed empirically, SGD always gives the results vibrating around an unsatisfied solution and faces a serious overfitting problem.  These challenges motivate us to develop a novel principle of variational Bayesian for massive stream dataset, which is different from the pioneer works \cite{2013BroderickSVI,2014TankSVI,2015McInerneyPPos,2013LinSVA}. Moreover, as our best knowledge, our approach is the first online variational inference approach to the problem of DDP task.

We summarize the contributions of this work as follows:
\begin{itemize}
	\item We develop a new Heterogeneous Matrix Factorization (HeMF) model for DDP task, which can efficiently capture the heterogeneous natural for Dyadic data in recommender systems. We construct HeMF in a Bayesian nonparametric manner and learn the number of communities and hidden dimensionality automatically during the implementation. HeMF model can be also viewed as a novel type of Bi-Cluster model, and it is the first method to introduce intrinsic linear structures into each bi-cluster.
	\item We develop a batch VB (bVB) inference method for learning the proposed HeMF model, which achieves better performance and converges faster compared with the state-of-the-art methods.
	\item We investigate an emerging problem in our paper, which is \emph{how to handle the massive data problem in DDP task}. This is an interesting problem in the field of DDP especially in the application of recommender system. Under the principle of Variational Bayesian, we derive a novel online Variational Bayesian (oVB) method to tackle this problem in a fraction of the time required by traditional inference. The performance of oVB is naturally guaranteed.
	\item We further develop an empirical Variational Bayesian (eVB) procedure to improve the performance of bVB and oVB. We apply both algorithms on real Dyadic data and the experimental results demonstrate that our methods significantly outperform the state-of-the-art approaches consistently. 
\end{itemize}

We organize the rest of this paper as follows, Sec.~2 reviews the background of Dirichlet Process for the construction of HeMF model; Sec.~3 introduces our HeMF model which is a Coupled Dirichlet Process Model to describe the generative process of factor vectors; Sec. 4 develops an efficient bVB inference method to infer the proposed HeMF model with batch of data; Sec. 5 derives a novel online learning method under a newly proposed online learning principle; Sec. 6 further proposes an empirical Variational Bayesian (eVB) to infer the hyper-parameters for bVB and oVB algorithm to improve their performance; to evaluate the performance of our method, Sec.~7 proposes extensive comparative studies of our approach with  previous ones on four real-world dataset, where our methods achieve superior performance over the competitive ones; and Sec.~8 concludes this paper.

\section{Dirichlet Process Mixture Model}
This section introduces the Dirichlet Process Mixture Model (DPMM), which is among the most popular clustering models in practice for analyzing the heterogeneous data. Different from the traditional parametric models for clustering, DPMM allows the number of groups to vary during inference, which provides great flexibility for exploratory analysis. 

We first introduce the Dirichlet Process (DP) \cite{1973FergusonDP}, which is typically denoted by $DP(\alpha, \mu)$ with a concentration parameter $\alpha$ and a base distribution $\mu$. A DP can be well constructed via the Chinese Restaurant Process~(CRP). Given a Chinese restaurant with countably infinite tables, customers walk in one after another and sit down at a certain table with the following scheme:
\begin{itemize}
	\item[1.] The first customer always chooses the first table.
	\item[2.] The $ t $th customer chooses either an unoccupied table with probability $ \dfrac{\alpha}{t + \alpha} $; or an occupied table with probability $ \dfrac{c}{t+\alpha} $, where $ c $ is the number of people sitting at that table. 
\end{itemize}
For the convenience of inference, Sethuraman~\cite{1991SethuramanDP} proposed a stick-breaking approach for DP construction, which can be defined as
\begin{equation}
\begin{split}
D &= \sum_{k=1}^{\infty} \pi_k \delta_{\phi_k},\\
\mbox{with}\hspace{2mm}\phi_k&\sim \mu, \forall k=1,2,\ldots,\\
\pi_k& = v_k\prod_{l=1}^{k-1}v_l, v_k\sim \mbox{Beta}(1,\alpha).\\
\end{split}
\end{equation}
Here, $ \mbox{Beta}(\alpha, \beta) $ denotes a Beta distribution with parameter $ \alpha $ and $ \beta $. 

The sample paths of a DP are almost sure discrete. Due to this nice property, the DP is widely used in the construction of mixture model which is very useful for modeling heterogeneous data. A DPMM can be well expressed with the following generative process
\begin{equation}
\begin{split}
&\hspace{10mm}D\sim DP(\alpha,\mu),\\
&\hspace{15mm}\theta_i\sim D,\\
&x_i\sim F(\cdot|\theta_i), \forall i=1,2,\ldots,n,
\end{split}
\end{equation}
where data $x_1,\ldots,x_n$ are the realizations from distribution $F$ with parameter $\theta_1,\ldots,\theta_n$. Because an atom $\phi_k$ can be repeatedly generated from $D$ with positive probability, there is a partition $\{\mathcal{P}_1,\ldots,\mathcal{P}_K\}$ of $\{1,\ldots,n\}$ such that $\theta_i$ are identical for all $i\in\mathcal{P}_k$, which we denote by $\phi_k$. 

The DPMM serves as a foundation for varieties of Bayesian nonparametric models, and has achieved substantial progress on representing feature-based data. We will use the DP and DPMM to construct our HeMF model for the DDP task.

\section{Heterogeneous Matrix Factorization with Coupled Dirichlet Process}

This section describes the Heterogeneous Matrix Factorization (HeMF) model for DDP, which utilizes a coupled Dirichlet Process to describe the generative model of a rating matrix.  Suppose we have $U$ users and $M$ items, and let $r_{ij}$ be the rating of user $ i $ for item $ j $, then we have:
\begin{equation}
\begin{split}
r_{ij} = a_i^Tb_j + e_{ij}.
\label{Eq_MFDef}
\end{split}
\end{equation}
Here, $b_j\in\mathbb{R}^{L\times 1}$ denotes an item-specific feature vector, where each element represents an objective score for one criterion, such as picture, directing, actor, actress, etc used in Academy Award of Merit; and  $a_i\in\mathbb{R}^{L\times 1}$ denotes the preference of user $ i $ for these criterion. Thus, $a_i^Tb_j$ presents the weighted average score of user $ i $ for item $ j $. Moreover, $e_{ij}$ denotes the zero-mean Gaussian noise with variance $\sigma^2$, which captures the uncertainty of rating behavior. Thus, the conditional distribution of the rating matrix $R\in\mathbb{R}^{U\times M}$ over $A$ and $B$ is given by 
\begin{equation}
p(R_{\Omega}|A,B,\sigma^2) = \prod_{(i,j)\in\Omega} \mathcal{G}(r_{ij}|a_i^Tb_j,\sigma^2),
\end{equation}
where $\mathcal{G}(x|\mu,\Sigma)$ denotes the Gaussian distribution with the mean vector $\mu$ and the covariance matrix $\Sigma$, and $\Omega$ denote the given observed index set . 

Traditional BPMF approaches are commonly developed with the assumption that $ a_i $ and $ b_j $ are Gaussian distributed. However, this assumption is not suitable for real commercial recommendation systems because of the heterogeneous natural of users and items. For example, in the commercial recommendation system \emph{Moivelens}, a movie can be categorized as \emph{computer animation, dramatic, touching or other labels}, while a user can be labeled by their specialties, such as \emph{historian, scientist, engineer, poet.} It is obvious that users from different specialties have their own preference for movies. Thus, the BPMF with Gaussian assumption is unsuitable for DDP task in the problem. It is significant to model such heterogeneous property for better prediction result. To tackle this issue, the HeMF model introduces the DPMM to model the $ a_i $ and $ b_j $ respectively resulting to a novel coupled Dirichlet Process to capture this heterogeneous natural.  

The motivation of applying Dirichlet Process in our paper is very intuitive and can be well explained from the perspective of CRP which is one of the construction of DP. Imaging a recommendation system with infinite number of interest communities, each with infinite capacity. The 1$^{st}$ user construct his own interest community with probability 1. At time $t+1$, a new user chooses at random to participate in one of the following $D+1$ interest communities: directly to the $d$-th already constructed community with probability $\frac{|d|}{t+1+\alpha}$ where $|d|$ is the size of $d$-th community, or establishes a new one with the probability $\frac{\alpha}{t+1+\alpha}$. 

For the convenience of inference, this paper uses the stick-breaking construction to build the HeMF and its generative process is as follows:

\begin{itemize}
	\item  For $i$-th user 
	\begin{easylist}
		& Draw $\pi_d\sim\mbox{Beta}(1,\alpha)$
		& Draw $\phi_d \sim H$
		& Draw $z_i = d\sim \pi_d\prod_{t=1}^{d-1}(1-\pi_t)$
		& Draw $a_i\sim F(\cdot|\phi_{z_i})$
	\end{easylist}
	\item  For $j$-th item
	\begin{easylist}
		& Draw $\omega_k\sim\mbox{Beta}(1,\beta)$
		& Draw $\psi_k \sim G$
		& Draw $\tilde{z}_j = k \sim \omega_k\prod_{t=1}^{k-1}(1-\omega_t)$
		& Draw $b_j\sim F(\cdot|\psi_{\tilde{z}_i})$
	\end{easylist}
	\item  Sample the rating value of their interaction
	\begin{easylist}
		& Draw $r_{ij}\sim\mathcal{G}(\cdot|a_i^Tb_j,\sigma^2)$.
	\end{easylist}
\end{itemize}

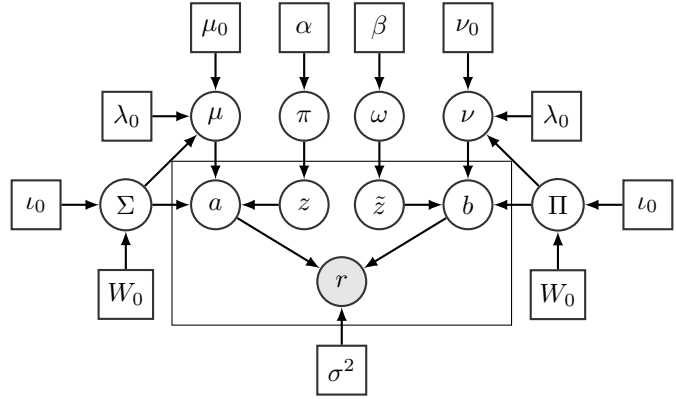
\begin{figure}[htp!]
	\centering
	\begin{tikzpicture}
	\tikzstyle{main}=[circle, minimum size = 6.5mm, thick, draw =black!80, node distance = 5mm]
	\tikzstyle{rect}=[rectangle, minimum size = 6.5mm, thick, draw =black!80, node distance =5mm]
	\tikzstyle{connect}=[-latex, thick]
	\tikzstyle{box}=[rectangle, draw=black!100]
	\node[rect] (e) {$\sigma^2$};
	\node[main, fill = black!10] (r) [above = of e] {$r$};
	\node[main] (z1) at (-0.5,2.2) {$z$};
	\node[main] (z2) at (0.5,2.2) {$\tilde{z}$};
	\node[main] (pi)[above = of z1] {$\pi$};
	\node[rect] (alpha)[above = of pi] {$\alpha$};
	\node[main] (omega)[above = of z2] {$\omega$};
	\node[rect] (beta)[above = of omega] {$\beta$};
	\node[main] (a) [left=of z1] { $a$};
	\node[main] (b) [right=of z2] {$b$};
	\node[main] (mua) [above=of a] {$\mu$ };
	\node[main] (mub) [above=of b] {$\nu$ };
	\node[main] (Sigmaa) [left=of a] {$\Sigma$ };
	\node[main] (Sigmab) [right=of b] {$\Pi$ };
	\node[rect] (mu0) [above=of mua] {$\mu_0$ };
	\node[rect] (nu0) [above=of mub] {$\nu_0$ };
	\node[rect] (lambdaa) [left=of mua] {$\lambda_0$ };
	\node[rect] (lambdab) [right=of mub] {$\lambda_0$ };
	\node[rect] (W0a) [below=of Sigmaa] {$W_0$ };
	\node[rect] (W0b) [below=of Sigmab] {$W_0$ };
	\node[rect] (iotaa) [left=of Sigmaa] {$\iota_0$ };
	\node[rect] (iotab) [right=of Sigmab] {$\iota_0$ };
	\path (alpha) edge [connect] (pi)
	(pi) edge [connect] (z1)
	(beta) edge [connect] (omega)
	(omega) edge [connect] (z2)
	(z1) edge [connect] (a)
	(z2) edge [connect] (b)
	(mua) edge [connect] (a)
	(mub) edge [connect] (b)
	(Sigmaa) edge [connect] (mua)
	(Sigmab) edge [connect] (mub)
	(Sigmaa) edge [connect] (a)
	(Sigmab) edge [connect] (b)
	(a) edge [connect] (r)
	(b) edge [connect] (r)
	(e) edge [connect] (r)
	(mu0) edge [connect] (mua)
	(nu0) edge [connect] (mub)
	(lambdaa) edge [connect] (mua)
	(lambdab) edge [connect] (mub)
	(W0a) edge [connect] (Sigmaa)
	(iotaa) edge [connect] (Sigmaa)
	(W0b) edge [connect] (Sigmab)
	(iotab) edge [connect] (Sigmab);
	\node[rectangle, inner sep=2.4mm,draw=black!100, fit= (z1) (r) (z2) (a) (b)] {};
	\end{tikzpicture}
	\caption{The graphical model for illustrating the proposed HeMF with Coupled Dirichlet Process, where the shadow node means the observed rating score, and the nodes covered by rectangle box mean the missing data. The leaf nodes illustrated in the graphical model mean the hyperparameters of the proposed HeMF model.}
	\label{fig:OurGraphical}
\end{figure}

It can be observed that there are two separate Dirichlet Processes for each user-item interaction, which can be better explained in Fig. \ref{fig:OurGraphical}.  Given $D$ communities for users and $K$ groups for items in total, an observed rating matrix $R$ is divided into $D\times K$ homogeneous sub-matrices, as shown in Fig. \ref{fig_ProDef}. The HeMF efficiently combine the division step and the matrix factorization step into a unified framework.
\subsection{Constructing Base Distributions}
Given the HeMF model, we further propose the base distribution $H$ and $G$ which should capture the previously mentioned low-rank properties. With our approach, the latent feature vectors from the same community should surround a low-rank space. Before stepping into the detailed constructions of $H$ and $G$, we first introduce the following lemma to construct a proxy matrix $Y$ for the target matrix $X$, whose columns coming from the same community.
\begin{lemma}
For a matrix $X\in\mathbb{R}^{m\times n}$ with $m\leq n$, its rank is bounded by the transferred matrix $Y = W^{\frac{1}{2}}(X-M)$, where $W\in\mathbb{R}^{m\times m}$ is a full-rank semi-definite positive matrix and each column in the matrix $M\in\mathbb{R}^{m\times n}$ is the mean of all columns in the matrix $X$. 
\label{lemma_rank}
\end{lemma}
\begin{proof}
Because each column in the matrix $M$ is the same with each other, we have $\mbox{rank}(M)=1$. And we have $\mbox{rank}(X-M)\leq \mbox{rank}(X)$ because all columns in both $X$ and $M$ can be represented by the columns in $X$. With the sub-additivity property of rank operation, we have
 \begin{equation}
 \begin{split}
  &\mbox{rank}(X)\leq \mbox{rank}(X-M)+\mbox{rank}(M).
  \end{split}
 \end{equation}
Thus, we can conclude that 
\begin{equation}
\mbox{rank}(X)-1 \leq \mbox{rank}(X-M)\leq \mbox{rank}(X).
\end{equation}
Because $W$ is a full-rank matrix, we achieve
\begin{equation}
\mbox{rank}(X)-1 \leq \mbox{rank}(Y)\leq \mbox{rank}(X),
\end{equation}
which means that the rank of the target matrix is bounded by the following inequalities,
\begin{equation}
\mbox{rank}(Y)\leq \mbox{rank}(X)\leq\mbox{rank}(Y)+1.
\end{equation}
\end{proof}
Thus, we can regularize the rank of the proxy matrix $Y$, instead of constraining the targeted matrix $X$ directly. An advantage of introducing the proxy matrix into our model stems from the facts that $Y$ provides much more flexibility when modeling $X$.  

The direct optimization  with rank regularization is clearly a NP-hard problem.  Fortunately, as proved in \cite{2010RechtNuCNorm}, the minimum rank approximation can be achieve by minimizing the nuclear norm. Thus, we can obtain,
\begin{equation}
\begin{split}
\min \mbox{rank}(X)\propto \min \mbox{rank}(Y) \propto\min \mbox{tr}(Y^TY).
\end{split}
\end{equation}
This approach has been widely used in machine learning community to promote a low rank solution without a pre-define $L$, and can be further formulated as
\begin{equation}
\begin{split}
\min \mbox{tr}(Y^TY) &\propto \min \sum_{t=1}^{n}(x_t-m)^TW(x_t-m)\\
&\propto \max -\frac{1}{2}\sum_{t=1}^{n}(x_t-m)^TW(x_t-m)\\
&\propto \max \ln \prod_{t=1}^{n}\mathcal{G}(x_t|m,W^{-1})
\end{split}
\end{equation}
with $m=\frac{1}{n}\sum_{t=1}^{n}x_t$. Thus, it can be found that regularizing the rank of $X$ is equivalent to assuming each column of the matrix to follow a Gaussian distribution. Thus, following \cite{2008SalakhutdinovBPMF}, for each user-item interaction, we assume $a_i$ and $b_j$ to be drawn from the following Gaussian distributions,
\begin{equation}
\begin{split}
 F(\cdot|\phi_{d}) &= \mathcal{G}(\cdot|\mu_{d},\Sigma_d),\\
 F(\cdot|\psi_{k}) &= \mathcal{G}(\cdot|\nu_{k},\Pi_k),
\end{split}
\label{Eq_GauPri}
\end{equation}
with $\phi_{d}=\{\mu_{d},\Sigma_d\}$ and $\psi_{k}=\{\nu_{k},\Pi_k\}$. As in \cite{2008SalakhutdinovBPMF,2010MackeyM3F}, we consider a fully Bayesian treatment where the topic model parameters $\phi_d$ and $\psi_k$ are the random realizations from the following distributions respectively,
\begin{equation}
\begin{split}
    H(\phi_{d}) &= \mathcal{G}(\mu_d|\mu_0,\lambda_0\Sigma_d)i\mathcal{W}(\Sigma_d|W_0,\iota_0),\\
	G(\psi_{k}) &= \mathcal{G}(\nu_k|\nu_0,\lambda_0\Pi_k)i\mathcal{W}(\Pi_k|W_0,\iota_0),
\end{split}
\end{equation}
where $i\mathcal{W}(\cdot)$ denotes the inverse Wishart distribution.

Thus, we can obtain the marginal distribution of sparsely observed rating matrix $R$ as follows:

\begin{equation}
\begin{split}
p(R_{\Omega}|\Xi) = \int &p(R_{\Omega},Y,\Theta|\Xi)dYd\Theta,
\end{split}
\label{Eq_HeterMFMAP}
\end{equation}
where $\Xi = \{\mu_0, \nu_0, \sigma^2,W_0,\iota_0,\lambda_0,\alpha, \beta\}$ denotes the hyper-parameters of the proposed model, $\Theta = \{\phi_d\}_{d=1}^D\cup\{\psi_k\}_{k=1}^K\cup\{\pi,\omega\}$ denotes the model parameters, and $Y= \{a_i,z_i\}_{i=1}^U\cup\{b_j,\tilde{z}_j\}_{j=1}^M $ denotes the missing data associated with each observed entry in the rating matrix $R$. The relationships between parameters has been illustrated in Fig. \ref{fig:OurGraphical}, where the observed rating $r$ is shown by a shaded node and the missing variables associated with the observed rating $r$ is covered by a rectangle box. The leaf nodes of this graphical model is shown by a rectangle node and denotes the a hyper-parameter of HeMF. Thus, $p(R_{\Omega},Y,\Theta|\Xi)$ can be decomposed as follows,
\begin{gather}
p(R_{\Omega},Y,\Theta|\Xi)=p(R_{\Omega}|A,B,\sigma^2)p(A|Z,\mu,\Sigma)p(B|\tilde{Z},\nu,\Pi)\nonumber\\
\cdot p(Z|\pi)p(\pi|\alpha)\prod_{d=1}^{D}\mathcal{G}(\mu_d|\mu_0,\lambda_0\Sigma_d)i\mathcal{W}(\Sigma_d|W_0,\iota_0)\nonumber\\
\cdot p(\tilde{Z}|\omega)p(\omega|\beta)\prod_{k=1}^{K}\mathcal{G}(\nu_k|\nu_0,\lambda_0\Pi_d)i\mathcal{W}(\Pi_k|W_0,\iota_0),\nonumber
\end{gather}
with
\begin{gather}
p(Z|\pi) =\prod_{i=1}^{U}\pi_{z_i}\prod_{t=1}^{z_i-1}(1-\pi_{t}),\hspace{1mm} p(\pi|\alpha)= \prod_{d=1}^{D} Beta(\pi_d|1,\alpha),\nonumber\\
p(\tilde{Z}|\omega) =\prod_{j=1}^{M}\omega_{\tilde{z}_j}\prod_{t=1}^{\tilde{z}_j-1}(1-\omega_{t}),\hspace{1mm} p(\omega|\beta)=\prod_{k=1}^{K} Beta(\omega_d|1,\beta),\nonumber\\
p(A|Z,\mu,\Sigma) =\prod_{i=1}^{U}\mathcal{G}(a_i|\mu_{z_i},\Sigma_{z_i}),\nonumber\\
p(B|\tilde{Z},\nu,\Pi) = \prod_{j=1}^{M}\mathcal{G}(b_j|\nu_{\tilde{z}_i},\Pi_{\tilde{z}_i}).\nonumber
\end{gather}

As shown in Eq. \ref{Eq_HeterMFMAP}, the Bayesian framework considers a whole class of models, rather than focusing on a single model to provide a solution to the DDP task. Thus, the HeMF with Bayesian framework can avoid overfitting problem by integrating out the parameters. 

\section{Variational Inference}
\label{sec_variInf}

Unfortunately, direct marginal inference of Eq. \ref{Eq_HeterMFMAP} is intractable. Markov chain Monte Carlo methods \cite{2010MackeyM3F} have been widely used to achieve exact marginal results, but typically require vast computational resources and become inefficient for complex models in high data dimensions. In this section, we derive a batch-based variational Bayesian inference method, which is a practical framework for Bayesian computations in graphical models.

Given the rating matrix $R$ observed on a sampled set of entries $\Omega$, Variational Bayesian (VB) approach introduces a trial distribution $q(Y,\Theta)$ to maximize the lower bound of the marginal distribution, $\mathcal{L}_{q(Y,\Theta)}(R_{\Omega}|\Xi)$, as follows,
\begin{equation}
\begin{split}
\ln & p(R_{\Omega}|\Xi) \geq \mathcal{L}_{q(Y,\Theta)}(R_{\Omega}|\Xi),\\
\end{split}
\end{equation}
where 
\begin{equation}
\begin{split}
\mathcal{L}_{q(Y,\Theta)}(R_{\Omega}|\Xi)&= \mbox{KL}(q(Y,\Theta)||p(R_{\Omega},Y,\Theta|\Xi)).
\label{Eq_VBObj}
\end{split}
\end{equation}
Here, $\mbox{KL}(\cdot)$ is the Kullback-Leibler (KL) distance between the trail distribution $q(Y,\Theta)$ and the joint distribution $p(R_{\Omega}|Y,\Theta,\Xi)p(Y,\Theta|\Xi)$. From Jensen's inequality, the lower bound of $\ln p(R_{\Omega}|\Xi)$ can be achieved by setting $q(Y,\Theta)\propto p(R_{\Omega}|Y,\Theta,\Xi)p(Y,\Theta|\Xi)$. In VB approximation, we often assume that $q(Y,\Theta)$ can be factorized as follows,
\begin{equation}
q(Y,\Theta) = q(Y)q(\pi)q(\omega)\prod_{d=1}^{D}q(\phi_d)\prod_{k=1}^{K}q(\psi_k),
\end{equation} 
where $q(Y)$, $q(\phi_d)$ and $q(\psi_k)$ can be further factorized as,
\begin{gather}
q(Y) = \prod_{i=1}^{U}q(a_i)q(z_i|a_i)\prod_{j=1}^{M}q(b_i)q(\tilde{z}_i|b_i),\nonumber\\
q(\phi_d)  = q(\mu_d)q(\Sigma_d|\mu_d),\hspace{2mm}q(\psi_k)  = q(\nu_k)q(\Pi_k|\nu_k).\nonumber
\end{gather}
This factorization corresponds to an approximation framework developed in physics called \emph{mean field theory} \cite{1988ParisiMeanField}, which breaks the entanglement among the model parameters, and leads to an efficient iterative algorithm.

Let $\theta_j$ denote the $j$-th parameter belonging to the set of parameters $Y\cup \Theta$. With other parameters fixed, the variational approximation problem with respect to $q(\theta_j)$ is equivalent to:
\begin{equation}
\begin{split}
\ln p(R_{\Omega}|\Xi)\geq&  \int q(\theta_j)\ln\frac{\exp\mathbb{E}_{\neq j}\ln p(R_{\Omega},Y,\Theta|\Xi)}{q(\theta_j)}d\theta_j,\\
\end{split}
\end{equation}
where $\mathbb{E}_{\neq j}(\cdot)$ denotes an expectation with respect to the posterior distributions $q$ over all parameters expect $\theta_j$. In this way, the close-form solution of $q(\theta_j)$ satisfies the following condition:
\begin{equation}
q(\theta_j)\propto \exp[\mathbb{E}_{\neq j}\ln p(R_{\Omega}|Y,\Theta,\Xi)p(Y,\Theta|\Xi)].
\label{Eq_PosterInf}
\end{equation} 
Thus, the marginal distribution $q(R_{\Omega}|\Xi)$ can be solved by alternatively calculating Eq. \ref{Eq_PosterInf} for each parameter. 

Actually, the computations of the posterior distribution over $\theta_j$ can be greatly reduced by considering the graphical model, as Fig. \ref{fig:OurGraphical}. When updating variational approximation result with respect to $\theta_j$, whose neighbor nodes are represented as $\theta_{\mathcal{N}_j}$, we can reformulate Eq. \ref{Eq_PosterInf} as
\begin{equation}
q(\theta_j)\propto \exp[\mathbb{E}_{q(\theta_{\mathcal{N}_j})}\ln p (R_{\Omega},\theta_j,\theta_{\mathcal{N}_j}|\Xi)],
\label{Eq_PosterInfSimp}
\end{equation}
where the number of involved parameters is dramatically decreased. For example, in the proposed HeMF model, the neighbor nodes of the stick-breaking vector $\pi$ contain only the hyper-parameter $\alpha$ and the membership parameter $z$. Thus, we only need to consider the prior distribution $p(\pi|\alpha)\prod_{i=1}^{U}p(z_i|\pi)$, as well as the expectation with respect to the associated posterior distribution $\prod_{i=1}^{U}q(z_i)$ in the derivation of posterior distribution over $\pi$. Thus, we have
\begin{equation}
\begin{split}
q(\alpha) \propto \exp[\mathbb{E}_{\prod_{i=1}^{U}q(z_i)}\ln p(\pi|\alpha)\prod_{i=1}^{U}p(z_i|\pi)].
\end{split}
\end{equation}
The detailed formulation of $q(\alpha)$ can be found in the later part of this section. 

The updating formulas associated with the users' taste matrix $A$ can be derived under the principle shown in Eq. \ref{Eq_PosterInfSimp}, while the updating formulas associated with the items' property matrix $B$ can be obtained by altering the notations. The parameters related to the users' taste matrix include $A$, $Z$, $\pi$ and $\{\mu_d,\Sigma_d\}_{d=1}^D$. To simplify our formulations, let $\langle f(x)\rangle = \int q(x)f(x)dx = \mathbb{E}_{q(x)}f(x)$, which denotes the expectation of $f(x)$ with respect to the distribution $q(x)$.

{\bf Estimation of the users' taste matrix $A$:} Based on the principle introduced in Eq. \ref{Eq_PosterInfSimp}, we can obtain that the posterior distribution over $a_i$ satisfies,
\begin{equation}
\begin{split}
\ln q(&a_i) \propto -\frac{1}{2}a^T_i(\frac{\sum_{(i,j)\in\Omega} \langle b_jb_j^T\rangle}{\sigma^2}+\sum_{d=1}^{D}q_i(d)\langle\Sigma_d^{-1}\rangle)a_i \\
&+(\frac{\sum_{(i,j)\in\Omega}r_{i,j}\langle b_j^T\rangle}{\sigma^2}+\sum_{d=1}^{D}q_i(d)\langle\mu_d^T\rangle\langle\Sigma_d^{-1}\rangle)a_i.\nonumber
\end{split}
\end{equation}
Since the parametric form of this posterior distribution is a quadratic function, the posterior distribution over $a_i$ is Gaussian with the following parameters:
\begin{equation}
\begin{split}
\langle a_i\rangle &=(\frac{1}{\sigma^2}\sum_{(i,j)\in\Omega} \langle b_jb_j^T\rangle+\sum_{d=1}^{D}q_i(d)\langle\Sigma_d^{-1}\rangle)^{-1}\\
&\cdot(\frac{1}{\sigma^2}\sum_{(i,j)\in\Omega}r_{i,j}\langle b_j\rangle+\sum_{d=1}^{D}q_i(d)\langle\Sigma_d^{-1}\rangle\langle\mu_d\rangle),\\
\langle a_i&a_i^T\rangle=(\frac{1}{\sigma^2}\sum_{(i,j)\in\Omega} \langle b_jb_j^T\rangle+\sum_{d=1}^{D}q_i(d)\langle\Sigma_d^{-1}\rangle)^{-1}\\
&\hspace{5mm}+\langle a_i\rangle\langle a_i^T\rangle.
\end{split}
\label{Eq_EStepUserTaste}
\end{equation}
From Eq. \ref{Eq_EStepUserTaste}, we find that the expectation of the $i$-th user's taste is jointly determined by rated items and the information of belonged communities of the user. It can be understood from our practical experience that a human's taste can be not only reflected by the choices of items, but also the communities he or she belongs to.

{\bf Estimation of the membership vector of users' taste $Z$:} After obtaining the updated users' taste matrix, we can get that the probability of each user's taste belonging to the $d$-th community satisfies
\begin{equation}
\begin{split}
\ln q(z_i=d) &\propto \gamma_i(d) = \xi_{1,i}(d)+\xi_{2,i}(d),
\end{split}
\label{Eq_EStepMem}
\end{equation}
with
\begin{equation}
\begin{split}
\xi_{1,i}(d) &=  -\frac{1}{2}(\mbox{tr}\langle a_ia_i^T\rangle\langle\Sigma_d^{-1}\rangle-2\langle\mu_d^T\rangle\langle\Sigma_d^{-1}\rangle\langle a_i\rangle\\
&+\mbox{tr}\langle\mu_d\mu_d^T\rangle\langle\Sigma_d^{-1}\rangle+\langle\ln|\Sigma_d|\rangle),\\
\xi_{2,i}(d)  &=\psi(\eta_{1,d}) - \psi(\eta_{1,d}+\eta_{2,d}) \\
&+ \sum_{j=1}^{d-1}[\psi(\eta_{2,j}) - \psi(\eta_{1,j}+\eta_{2,j})],\\
\end{split}
\label{Eq_EStepMemDetail}
\end{equation}
where $\psi(\cdot)$ denotes the digamma function \cite{1964AbramowitzHandbook}, $\eta_{1,d}$ and $\eta_{2,d}$ denote the parameters to describe the posterior distribution over the $d$-th element in the vector $\pi$.
Denoting $q(z_i=d)$ as $q_i(d)$, we can update the indicator variable $z_i$ as follows:
\begin{equation}
\begin{split}
q_i(d) &= \frac{\exp(\gamma_i(d))}{\sum_{j=1}^{D}\exp(\gamma_i(j))}.
\end{split}
\label{Eq_EStepMember}
\end{equation}

As observed in Eq. \ref{Eq_EStepMem}, the computation of $\gamma_i(d)$ consists of two terms, $\xi_{1,i}(d)$ and $\xi_{2,i}(d)$. Let $\xi_{2,i}(d)=0$, we can find that $q_i(d)$ is the same with E-step in the classical Expectation-Maximization (EM) method \cite{1977DempsterEM}. Represented by the stick-breaking process, the effect of Dirichlet Process is introduced into our model by $\xi_{2,i}(d)$. The strategy of slightly revising the E-step to improve the performance of EM algorithm has been heuristically explored in some pioneer works, e.g. RPCL \cite{2002XuRPCL} and DAEM \cite{1998UedaDAEM}.

{\bf Estimation of the Beta vector $\pi$:} After the inference of missing data associated with the users' taste matrix, we can first solve the posterior distributions over the Beta realizations $\{\pi_d\}_{d=1}^D$ used in the stick-breaking construction of Dirichlet Process. Under the principle in Eq. \ref{Eq_PosterInfSimp}, we have 
\begin{equation}
\begin{split}
\ln q(\pi) &\propto \ln \prod_{d=1}^{D}\pi_d^{\rho_d}(1-\pi_d)^{\alpha+\sum_{j=d+1 }^{D}\rho_j-1},
\end{split}
\end{equation}
where $\rho_d = \sum_{i=1}^{U}q_i(d)$ represents the expected number of users coming from the $d$-th community. It is obvious that each element in the Beta vector is independent with each other and beta distributed. For $\pi_d$, the parameters of its posterior distribution is as follows:
\begin{equation}
\begin{split}
\eta_{1,d} & = \rho_d + 1,\hspace{2mm}\eta_{2,d}  =\alpha+\sum_{j=d+1}^{D} \rho_j.
\end{split}
\label{Eq_EStepPi}
\end{equation}
Thus, the expected value of the $d$-th stick length is $\frac{\rho_d + 1}{\sum_{j=d}^{D} \rho_j+1+\alpha}$, where $\alpha$ provides a basis term. 

{\bf Estimation of the $d$-th community mean vector $\mu_d$:} According to the principle in Eq. \ref{Eq_PosterInfSimp}, we can get that the posterior distribution over $\mu_d$ satisfies the following formula, 
\begin{equation}
\begin{split}
\ln q(&\mu_d)\propto -\frac{1}{2}\mu_d^T(\sum_{i=1}^{U}q_i(d)\langle\Sigma_d^{-1}\rangle+\frac{\langle\Sigma_d^{-1}\rangle}{\lambda_0})\mu_d\\
& + (\sum_{i=1}^{U}q_i(d)\langle a_i^T\rangle\langle\Sigma_d^{-1}\rangle+\langle\mu_0^T\rangle\frac{\langle\Sigma_d^{-1}\rangle}{\lambda_0})\mu_d.
\end{split}
\end{equation}
Similar to the inference of users' taste matrix $A$, the posterior distribution of $\mu_d$ is also Gaussian distributed with:

\begin{equation}
\begin{split}
\langle\mu_d\rangle & =\frac{\lambda_0}{\lambda_0+\rho_d}\sum_{i=1}^{U}q_i(d)\langle a_i\rangle+\frac{1}{\lambda_0+\rho_d}\mu_0,\\
\langle\mu_d&\mu_d^T\rangle  =  \frac{\lambda_0}{\lambda_0+\rho_d}\langle\Sigma_d^{-1}\rangle^{-1}+\langle\mu_d\rangle\langle\mu_d^T\rangle.
\end{split}
\label{Eq_EMu}
\end{equation}
It can be observed that the expectation of the $d$-th mean vector, $\langle\mu_d\rangle$, consists of two terms corresponding to the weighted mean of all users' taste vectors, and the prior knowledge to avoid the over-fitting problem respectively. Eq. \ref{Eq_EMu} states that the properties of each community is fully determined by the users belonging to it.

{\bf Estimation of the $d$-th community covariance matrix $\Sigma_d$:} Finally, we introduce our approach to inference the second-order statistics of the $d$ community, denoted as the covariance matrix $\Sigma_d$. Following Eq. \ref{Eq_PosterInfSimp}, we have,
\begin{equation}
\begin{split}
&\hspace{4mm}\ln q(\Sigma_d|\mu_d) \propto   -\frac{1}{2}\mbox{tr}\sum_{i=1}^{U}q_i(d)(a_i-\mu_d)(a_i-\mu_d)^T\Sigma_d^{-1}\\
&-\frac{1}{2}\sum_{i=1}^{U}q_i(d)\ln |\Sigma_d|-\frac{1}{2}\mbox{tr}(W_0\Sigma_d^{-1}) - \frac{\iota_0 +r + 1}{2}\ln |\Sigma_d|\\
&\hspace{11mm}-\frac{1}{2}[\mbox{tr}\frac{(\mu_d-\mu_0)(\mu_d-\mu_0)^T}{\lambda_0}{\Sigma_d^{-1}}+ \ln |\Sigma_d|].\\
\end{split}
\end{equation}
Thus, the posterior distribution of $\Sigma_d$ is also an inverse-Wishart distribution $i\mathcal{W}(W_d,\iota_d)$ with following parameters:
\begin{equation}
\begin{split}
W_d & = \sum_{i=1}^{U}q_i(d)(\langle a_ia_i^T\rangle-2\langle\mu_d\rangle\langle a_i^T\rangle + \langle\mu_d\mu_d^T\rangle)\\
&+  \frac{1}{\lambda_0}(\langle\mu_d\mu_d^T\rangle-2\mu_0\langle\mu_d^T\rangle+\mu_0\mu_0^T)+W_0,\\
\iota_d & = \iota_0 + \rho_d,
\end{split}
\end{equation}
Thus, we can subsequently get
\begin{equation}
\begin{split}
\langle\Sigma_d^{-1} \rangle &= \iota_dW_d^{-1},\\
\langle\ln |\Sigma_d|\rangle & = \ln|\phi_d|-\psi(\frac{\iota_d}{2})-r\ln(2).
\end{split}
\label{Eq_EStepCovariance}
\end{equation}

\begin{algorithm}
	\begin{algorithmic}[1]
		\renewcommand{\algorithmicrequire}{\textbf{Input:}}
		\renewcommand{\algorithmicensure}{\textbf{Output:}}
		\REQUIRE A rating matrix $R$ observed on positions $\Omega$, and the hyper-parameters $\Xi$
		\ENSURE  The completed matrix $\hat{R}$
		\\ \textit{Initialization} :
		\WHILE{not converge}
		\STATE Estimate the users' taste matrix $A$ by Eq. \ref{Eq_EStepUserTaste}
		\STATE Estimate the membership vector of users' taste $Z$ \\by Eq. \ref{Eq_EStepMember}
		\STATE Estimate the Beta vector $\pi$ by Eq. \ref{Eq_EStepPi}
		\STATE Estimate the $d$-th community mean vector $\mu_d$ by \\Eq. \ref{Eq_EMu}
		\STATE Estimate the $d$-th community covariance matrix $\Sigma_d$ by Eq. \ref{Eq_EStepCovariance}
		\ENDWHILE
		\RETURN $\hat{R} = \langle{A}\rangle^T\langle{B}\rangle$ 
	\end{algorithmic} 
	\caption{Batch Variational Bayesian (bVB) Inference for HeMF}
	\label{Alg_VarHeMF}
\end{algorithm}

By iteratively updating formulas from Eq. \ref{Eq_EStepUserTaste} to \ref{Eq_EStepCovariance}, we develop an inference method for the proposed HeMF model, which can achieve at least a local optimal solution. The procedure is summarized in the Alg. \ref{Alg_VarHeMF}. Noted that our inference method will degenerate to the famous Bayesian Probabilistic Matrix Factorization (BPMF) model \cite{1999HofmannBPMF} by setting $K=D=1$. At each step in Alg. \ref{Alg_VarHeMF}, we update the selected parameter with other parameters fixed. Thus, the lower bound of the marginal distribution is maximized successively in our method, whose convergence is naturally guaranteed.

\section{Online Variational Inference}

The method introduced in the last section requires to load the whole dataset for inference. However, when dealing with massive data, which becomes common nowadays, the memory space may be not enough for the bVB inference. To tackle this problem, we now develop an online variational inference method. We assume that a local optimal configuration has been achieved to approximate the marginal distribution of observing the rating matrix $R_{\Omega}$ at time $t$, and there are totally $D$ user communities and  $K$ item groups. If new scores are observed at entires $\Omega'$ at time $t+1$, the marginal distribution of observing the rating matrix $R_{\Omega\cup\Omega'}$ can be reformulated as:
\begin{gather}
\begin{split}
p &(R_{\Omega\cup \Omega'}|\Xi)=\int dYd\Theta_{D+1,K+1}p(Y,\Theta_{D+1,K+1}|\Xi)\\
&\cdot p(R_{\Omega}|Y,\Theta_{D+1,K+1},\Xi)p(R_{\Omega'}|Y,\Theta_{D+1,K+1},\Xi).
\end{split}
\label{Eq_SeqVBMarg}
\end{gather}
It can be seen that the computation of marginal distribution for the newly observed rating matrix $R_{\Omega\cup\Omega'}$ consists of three parts, the prior distribution $p(Y,\Theta_{D+1,K+1}|\Xi)$, the likelihood of observing the rating matrix $R_{\Omega}$ and the likelihood of observing the rating matrix $R_{\Omega'}$. In Eq. \ref{Eq_SeqVBMarg}, we set $\Theta_{D+1,K+1}$ as the integral term, which provides our model the capability to generate new components on the fly. To approximate the marginal distribution $p(R_{\Omega\cup \Omega'}|\Xi)$, we assume that a new trial distribution $q^{t+1}(Y,\Theta_{D+1,K+1})$ can be found over the restricted function, as VB,
\begin{gather}
\begin{split}
q^{t+1}(Y,&\Theta_{D+1,K+1})=  q^{t+1}(Y)q^{t+1}(\pi)q^{t+1}(\omega)\\
&\cdot\prod_{d=1}^{D+1}q^{t+1}(\phi_d)\prod_{k=1}^{K+1}q^{t+1}(\psi_k).
\end{split}
\end{gather}

Let $\theta_j$ denote the $j$-th parameter belonging to the set of parameters $Y\cup \Theta_{D+1,K+1}$. Given the new observations sampled on the positions $\Omega'$, the variational approximation problem with respect to $q(\theta_j)$ is equivalent to 
\begin{equation}
\begin{split}
\ln& p(R_{\Omega\cup\Omega'}|\Xi)\\
&\geq  \int q^{t+1}(\theta_j)\mathbb{E}_{\neq j}\ln p(R_{\Omega}|Y,\Theta_{D+1,K+1},\Xi)d\theta_j\\
&+\int q^{t+1}(\theta_j)\mathbb{E}_{\neq j}\ln p(R_{\Omega'}|Y,\Theta_{D+1,K+1},\Xi)d\theta_j\\
&+\int q^{t+1}(\theta_j)\mathbb{E}_{\neq j}\ln p(Y,\Theta_{D+1,K+1}|\Xi)d\theta_j\\
&-\int q^{t+1}(\theta_j)\ln q^{t+1}(\theta_j)d\theta_j.\\
\end{split}
\end{equation}

In this way, the close-form solution of $q(\theta_j)$ satisfies the following condition:
\begin{equation}
\begin{split}
&q^{t+1}(\theta_j)\propto\exp[\mathbb{E}_{\neq j}\ln p(R_{\Omega'}|Y,\Theta_{D+1,K+1},\Xi)]\\
&\hspace{14.5mm}\cdot \exp[\mathbb{E}_{\neq j}\ln p(R_{\Omega},Y,\Theta_{D+1,K+1}|\Xi)]\\
&\propto q^{t}(\theta_j)\exp[\mathbb{E}_{\neq j}\ln p(R_{\Omega'}|Y,\Theta_{D+1,K+1},\Xi)]
\end{split}
\label{Eq_PosterInfOnline}
\end{equation} 
Compared with the traditional Variational Bayesian updating formula (\ref{Eq_PosterInf}), there is one additional term in our online inference algorithm, which introduces the influence of new observations into the updating formula. There are some related works proposed for deploying the principle of variational Bayesian for large-scale datasets, including stochastic variational inference (SVI$^1$) \cite{2013BroderickSVI}, stream variational inference (SVI$^2$) \cite{2014TankSVI}, Population Variational Bayesian (pVB) \cite{2015McInerneyPPos} and sequential variation approximation (SVA) \cite{2013LinSVA}. Compared with these seminal works, our method is different by using the Bayesian decomposition of the joint distribution, which is simple but efficient. Moreover, to our best knowledge, no previous online variational method has been applied to the DDP task.

As illustrated in Fig. \ref{fig_ProDef}, new entries can be observed in the rating matrix $R$ as time goes on. Thus, the model should be updated sequentially. In this section, our focus mainly concentrates on the mathematical derivation of the parameters associated with the users' taste matrix $A$, as Sec. \ref{sec_variInf}. Assuming we have achieved a local optimal solution for the observed rating matrix $R_{\Omega}$, we should update the parameters by following formulas in each single pass with the newly observed entries $\Omega'$.

{\bf Online estimation of the users' taste matrix $A$:} As proved in Sec. \ref{sec_variInf}, the posterior distribution over $a_i$ is also a Gaussian distribution. Thus, under the principle of online variational Bayesian (oVB), as shown in Eq. \ref{Eq_PosterInfOnline}, we derive the updating formulas over $a_i$ for different cases as follows,
\begin{itemize}
	\item[1] If the $i$-th user rates new items, we get 	
	\begin{equation}
	\begin{split}
	&\langle a_i\rangle^{t+1} = \langle a_i\rangle^{t}+\frac{1}{\sigma^2}\Sigma_{a_i}^{t}\Delta_{1,i}^{1+t} \\
	&\hspace{2mm}- \Sigma_{a_i}^{t}(\sigma^2{\bf I}+\Delta_{2,i}^{t+1}\Sigma_{a_i}^{t})^{-1}\Delta_{2,i}^{t+1} (\langle a_i\rangle^{t}+\frac{1}{\sigma^2}\Sigma_{a_i}^{t}\Delta_{1,i}^{1+t}),\\
	&\Sigma_{a_i}^{t+1} =  \Sigma_{a_i}^{t} - \Sigma_{a_i}^{t}(\sigma^2{\bf I}+\Delta_{2,i}^{t+1}\Sigma_{a_i}^{t})^{-1}\Delta_{2,i}^{t+1}\Sigma_{a_i}^{t}.\\
	\end{split}
	\label{Eq_OnlineUserTaste}
	\end{equation}
	where  $\Delta_{1,i}^{1+t}=\sum_{(i,j)\in\Omega'}r^{1+t}_{i,j}\langle b_j\rangle^{t}$, and $\Delta_{2,i}^{t+1} = \sum_{(i,j)\in\Omega'} \langle b_jb_j^T\rangle^{t}$ summarize the information of newly rated items. $\Sigma_{a_i}$ denotes the covariance matrix of the posterior distribution over $a_i$, and it is equivalent to $\Sigma_{a_i} = \langle a_ia_i^T\rangle-\langle a_i\rangle\langle a_i\rangle^T $.  As we can observe,  the expected user's taste vector will be updated with the newly rated items.
	
	\item[2] If the $i$-th user changes his previous rating scores, we should update the user's taste by Eq. \ref{Eq_OnlineUserTaste}, but with $\Delta_{1,i}^{1+t}=\sum_{(i,j)\in\Omega'}(r^{t+1}_{i,j}-r^{t}_{i,j})\langle b_j\rangle^{t}$ and $\Delta_{2,i}^{1+t}$ remaining the same form. It can be understood practically that the re-rated item will enhance its weight in describing the user's taste if this user shows higher preference to this object. 
	
	\item[3] If someone registers a new account in the recommendation system and rates some items, the user's taste vector can by initialized by Eq. \ref{Eq_EStepUserTaste} directly, which solves the problem of \emph{cold start} naturally.
\end{itemize} 

Such updating problem has been previously tackled by Stochastic Gradient Descent (SGD) methods for matrix factorization. Starting with some initial value $ a_i^{t}$, SGD refines the parameter value by iterating the stochastic difference equation as follows:
\begin{equation}
 a_i^{t+1} = a^{t} - \epsilon^t\mathcal{L}'(a^{t}),
\end{equation}
where $\epsilon^{[t]}$ denotes a sequence of decreasing step sizes and $- \mathcal{L}'(a^{t})$ is the direction of steepest descent given the new observations. Compared with SGD, our method gets a local optimal solutions for $\langle a_i\rangle$ and $\Sigma_{a_i}^{t+1}$ at each iterative step. As shown empirically, it usually yields a faster convergence speed without the need of tuning the learning rate. Moreover, under the framework of Bayesian theory, we can integrate the prior knowledges into our model and avoid the overfitting problem by considering a family of models.

{\bf Online estimation of the membership vector $Z$:} With Eq. \ref{Eq_OnlineUserTaste}, we can update the user's taste with new observations. Before updating the membership of user $ i $, let
\begin{equation}
\begin{split}
\langle\Sigma_0^{-1} \rangle &= \iota_0W_0^{-1},\\
\langle\ln |\Sigma_0|\rangle & = \ln|W_0|-\psi(\frac{\iota_0}{2})-r\ln(2),
\end{split}
\end{equation}
which denote some initial parameters of the users' community. Thus, the probability of introducing a new component for $\langle a_i\rangle^{t+1}$ satisfies 
\begin{equation}
\begin{split}
\ln &q(z_i=D+1) \propto  \gamma_i(D+1)=\xi_{1,i}(0)+\xi_{2,i}(0),\\
\end{split}
\end{equation}
where $\xi_{1,i}(0),\xi_{2,i}(0)$ share the same forms with Eq. \ref{Eq_EStepMemDetail}. If we allow the online scheme to introduce a new component for every updating pass, there will be infinite components asymptotically, which is unnecessary. Empirically, $\gamma_i(D+1) $ is negligible enough for most passes, which indicates that the updated user's tastes can be adequately explained by existing communities, and there is no need of new community. In practice, we set a small value $\epsilon$ and increase $D$ to $D+1$ only when $ \gamma_i(D+1) >\epsilon$. This simple strategy is very efficient in controlling the model size in practical applications.

Thus, the indicator variable can be updated as follows:
\begin{equation}
\begin{split}
q_i(d) &= \frac{\exp(\gamma_i(d))}{\sum_{t=1}^{D}\exp(\gamma_i(t))},
\end{split}
\label{Eq_OnlineMember}
\end{equation}
which shares the same formulation with Eq. \ref{Eq_EStepMem}, besides the fact that one more community may be generated into the learning process on the fly.

To simplify our derivations, let 
\begin{equation}
\begin{split}
\vartheta^{t+1}_1(d) &= \sum_{i\in\Omega'}q^{t+1}_i(d)-q^{t}_i(d),\\
\vartheta^{t+1}_2(d) &= \sum_{i\in\Omega'}q^{t+1}_i(d)\langle a_i\rangle^{t+1}-q^{t}_i(d)\langle a_i\rangle^{t},\\
\vartheta^{t+1}_3(d) &= \sum_{i\in\Omega'}q^{t+1}_i(d)\langle a_ia_i^T\rangle^{t+1}-q^{t}_i(d)\langle a_ia_i^T\rangle^{t},
\end{split}
\end{equation}
which denotes the updates of the missing data for the $i$-th user's taste. For a newly registered user, we have $q_{i}({d})=0,\forall i\in\{1,\ldots,D\}$. Next, we will move on to introduce the updating of model parameters, $\Theta$.

{\bf Online estimation of the Beta vector $\pi$:} It has been known that the posterior distribution over the $d$-th entry in the vector $\pi$ is also a beta distribution with form $\mbox{Beta}(\eta_{1,d},\eta_{2,d})$. Under the principle of oVB, the beta distribution parameters can be updated as follows,
\begin{equation}
	\begin{split}
		&\hspace{3mm}\eta^{t+1}_{1,d}  = \eta^{t}_{1,d}+\vartheta^{t+1}_1(d),\\
		&\eta^{t+1}_{2,d}  =\eta^{t}_{2,d}+ \sum_{j=d+1}^{k} \vartheta^{t+1}_1(j).
	\end{split}
	\label{Eq_OnlinePi}
\end{equation}
It can be found that the posterior distribution parameters are adjusted by the updating term of membership vector. 

{\bf Online estimation of the $d$-th community mean vector $\mu_d$:} Under the principle in Eq. \ref{Eq_PosterInfOnline}, the posterior of $\mu_d$ is also a Gaussian distribution with parameters:

\begin{equation}
\begin{split}
\langle\mu_d\rangle^{t+1} & =\langle\mu_d\rangle^{t}-\frac{\vartheta^{t+1}_1(d)\langle\mu_d\rangle^{t}-\lambda_0\vartheta^{t+1}_2(d)}{\lambda_0+\rho_d^{t+1}},\\
\end{split}
\label{Eq_OnlineMu}
\end{equation}
with $\rho_d^{t+1} = \rho_d^{t} + \vartheta^{t+1}_1(d)$, denoting the updated expected number of users in the $d$-th community. Additionally, we have
\begin{equation}
	\begin{split}
		\langle\mu_d\mu_d^T\rangle^{t+1} & =  \frac{\lambda_0}{\lambda_0+\rho_d^{t+1}}{\langle\Sigma_d^{-1}\rangle^{t}}^{-1}+\langle\mu_d\rangle^{t+1}\langle\mu_d^T\rangle^{t+1}.
	\end{split}
\end{equation}

To simplify the formulations for sequential estimation of $\Sigma_d$, let
\begin{equation}
\begin{split}
&\hspace{12mm}\vartheta^{t+1}_4(d) = \langle\mu_d\rangle^{t+1} - \langle\mu_d\rangle^{t},\\
&\hspace{10mm}\vartheta^{t+1}_5(d) = \langle\mu_d\mu_d^T\rangle^{t+1}- \langle\mu_d\mu_d^T\rangle^{t},\\
&\vartheta^{t+1}_6(d) = \langle\mu_d\rangle^{t+1}\sum_{i\in\Omega'}\langle a_i^T\rangle^{t+1}-\langle\mu_d\rangle^{t}\sum_{i\in\Omega'}\langle a_i^T\rangle^{t},
\end{split}
\end{equation} 
which represent the updates of the $d$-th community mean vector.
 
{\bf Online estimation of the $d$-th community covariance matrix $\Sigma_d$:} The posterior distribution of $\Sigma_d$ is also an inverse-Wishart distribution, $i\mathcal{W}(W_d^{t+1},\iota^{t+1}_d)$ with:
\begin{equation}
	\begin{split}
		W_d^{t+1}  = &W_d^{t}- \frac{2}{\lambda_0}{\mu_0\vartheta^{t+1}_4(d)}-2\vartheta^{t+1}_6(d)\\
		&+ (U+\frac{1}{\lambda_0}) \vartheta^{t+1}_5(d)+\vartheta^{t+1}_3(d),\\
		\iota^{t+1}_d  = &\iota^{t}_d +\vartheta^{t+1}_1(d),
	\end{split}
\end{equation}
Thus, we can subsequently get
\begin{equation}
	\begin{split}
	    &\hspace{13mm}\langle\Sigma_d^{-1} \rangle^{t+1} ={ W_d^{t+1}}^{-1}\iota^{t+1}_d,\\
		&\langle\ln |\Sigma_d|\rangle^{t+1}  = \ln| W_d^{t+1}|-\psi(\frac{\iota^{t+1}_d}{2})-r\ln(2).
	\end{split}
	\label{Eq_OnlineCov}
\end{equation}

\begin{algorithm}
	\begin{algorithmic}[1]
		\renewcommand{\algorithmicrequire}{\textbf{Input:}}
		\renewcommand{\algorithmicensure}{\textbf{Output:}}
		\REQUIRE The new entries observed on positions $\Omega'$, the hyper-parameters $\Xi$, and the local optimal solution for the posterior distribution $q^t(Y,\Theta)$
		\ENSURE  The completed matrix $\hat{R}$
		\\ \textit{Initialization} :
		\STATE Update the users' taste matrix $A$ by Eq. \ref{Eq_OnlineUserTaste}
		\STATE Update the membership vector of users' taste $Z$ \\by Eq. \ref{Eq_OnlineMember}
		\STATE Update the Beta vector $\pi$ by Eq. \ref{Eq_OnlinePi}
		\STATE Update the $d$-th community mean vector $\mu_d$ by \\Eq. \ref{Eq_OnlineMu}
		\STATE Update the $d$-th community covariance matrix $\Sigma_d$ by Eq. \ref{Eq_OnlineCov}
		\RETURN $\hat{R} = \langle{A}\rangle^T\langle{B}\rangle$ 
	\end{algorithmic} 
	\caption{online Variational Bayesian (oVB) inference for HeMF}
	\label{Alg_OnlVarHeMF}
\end{algorithm}
\label{Sec_SeqVar}

In summary, given the newly observed entries, we can update the proposed HeMF model by successively executing the formulas from Eq. \ref{Eq_OnlineUserTaste} to Eq. \ref{Eq_OnlineCov}. The detailed procedure can be found in Alg. \ref{Alg_OnlVarHeMF}. It can be observed that the whole sequential updating procedures are triggered by the update of user's taste. Under the principle of Variational Bayesian, we update the selected user's taste with other parameters fixed.

\subsection{Notes about oVB}
The previous section discussed the procedure to update the given model configuration when there are newly observed entries. Actually, the online updating procedure, Alg. \ref{Alg_OnlVarHeMF}, should be triggered every time when some parameters change. To help understand the proposed oVB, this section comprehensively explains the updating behaviors of parameters when there are the newly observed rating entries.

 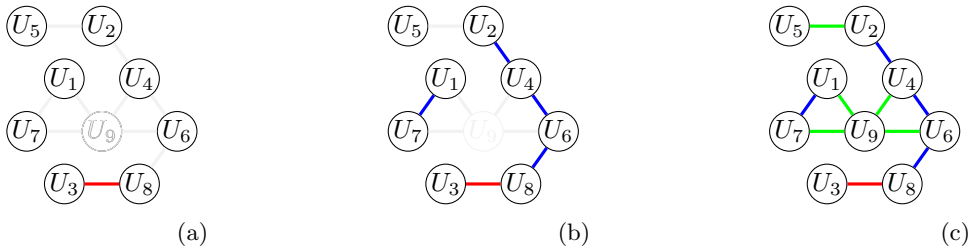
\begin{figure}[h!] 
 	\centering
 	\begin{subfigure}[b]{0.3\textwidth}
 		\begin{tikzpicture}
 		\tikzstyle{vertex}=[circle, draw, minimum size=15pt,inner sep=0pt]
 		
 		\foreach \name/\x in {U_5/1, U_2/2}
 		\node[vertex] (G-\name) at (\x,0) {$\name$};
 		
 		\foreach \name/\x in {U_1/1.5, U_4/2.5}
 		\node[vertex] (G-\name) at (\x,-0.7) {$\name$};
 		
 		\foreach \name/\x in {U_7/1, U_9/2, U_6/3}
 		\node[vertex] (G-\name) at (\x,-1.4) {$\name$};
 		
 		\node[vertex,draw=gray!10] (G-U_9) at (2,-1.4) {$\color{gray!10}U_9$};
 		
 		\foreach \name/\x in {U_3/1.5, U_8/2.5}
 		\node[vertex] (G-\name) at (\x,-2.1) {$\name$};
 		
		\foreach \from/\to in {U_3/U_8}
		\draw[red,very thick] (G-\from)--(G-\to);
 		\foreach \from/\to in {U_5/U_2,U_2/U_4,U_1/U_7,U_1/U_9,U_7/U_9,U_9/U_4,U_6/U_4,U_6/U_9,U_6/U_8}
 		\draw[gray!10,very thick] (G-\from)--(G-\to);
 		
 		\end{tikzpicture}
 		\subcaption{}
 	\end{subfigure}
 	\begin{subfigure}[b]{0.3\textwidth}
 		\begin{tikzpicture}
 		\tikzstyle{vertex}=[circle,draw,minimum size=15pt,inner sep=0pt]
 		
 		\foreach \name/\x in {U_5/1, U_2/2}
 		\node[vertex] (G-\name) at (\x,0) {$\name$};
 		
 		\foreach \name/\x in {U_1/1.5, U_4/2.5}
 		\node[vertex] (G-\name) at (\x,-0.7) {$\name$};
 		
 		\foreach \name/\x in {U_7/1, U_6/3}
 		\node[vertex] (G-\name) at (\x,-1.4) {$\name$};
 		
 		\node[vertex,draw=gray!10] (G-U_9) at (2,-1.4) {$\color{gray!10}U_9$};
 		
 		\foreach \name/\x in {U_3/1.5, U_8/2.5}
 		\node[vertex] (G-\name) at (\x,-2.1) {$\name$};
 			
		\foreach \from/\to in {U_3/U_8}
		\draw[red,very thick] (G-\from)--(G-\to);
 		\foreach \from/\to in {U_1/U_7,U_6/U_4,U_2/U_4,U_6/U_8}
 		\draw[blue,very thick] (G-\from)--(G-\to);
 		\foreach \from/\to in {U_5/U_2,U_1/U_9,U_7/U_9,U_9/U_4,U_6/U_9}
 		\draw[gray!10,very thick] (G-\from)--(G-\to);
 		
 		\end{tikzpicture}
 		\subcaption{}
 	\end{subfigure}
 	\begin{subfigure}[b]{0.3\textwidth}
 		\begin{tikzpicture}
 		\tikzstyle{vertex}=[circle,draw,minimum size=15pt,inner sep=0pt]
 		
 		\foreach \name/\x in {U_5/1, U_2/2}
 		\node[vertex] (G-\name) at (\x,0) {$\name$};
 		
 		\foreach \name/\x in {U_1/1.5, U_4/2.5}
 		\node[vertex] (G-\name) at (\x,-0.7) {$\name$};
 		
 		\foreach \name/\x in {U_7/1, U_9/2, U_6/3}
 		\node[vertex] (G-\name) at (\x,-1.4) {$\name$};
 		
 		\foreach \name/\x in {U_3/1.5, U_8/2.5}
 		\node[vertex] (G-\name) at (\x,-2.1) {$\name$};
 		
		\foreach \from/\to in {U_3/U_8}
		\draw[red,very thick] (G-\from)--(G-\to);
		\foreach \from/\to in {U_1/U_7,U_6/U_4,U_2/U_4,U_6/U_8}
		\draw[blue,very thick] (G-\from)--(G-\to);		
		\foreach \from/\to in {U_5/U_2,U_1/U_9,U_7/U_9,U_9/U_4,U_6/U_9}
		\draw[green,very thick] (G-\from)--(G-\to);
 		
 		\end{tikzpicture}
 		\subcaption{}
 	\end{subfigure}
 	\caption{The evolution of users' graph, where two users will be linked if they rate the same items. (a), (b) and (c) denote the users' graphs at time $t$, $t+1$ and $t+2$ respectively.}
 	\label{fig_updatingflow}
 \end{figure}
 
 As illustrated in Fig. \ref{fig_updatingflow}, we have constructed a graph for the sparsely observed matrix shown in Fig. \ref{fig_ProDef}, where two users are linked when they have rated same items. It can be observed that at time $t$, only the 3$^{rd}$ user and the 8$^{th}$ user are linked. At time $t+1$, the 6$^{th}$ user rated the 5$^{th}$ item, thus a link between the 4$^{th}$ user and the 6$^{th}$ user is established. At time $t+2$, it can be found that there are no more isolated nodes in the graph, resulting to a connected graph. With these graphs as an example, we will explain how our online learning algorithm behaves.
 
At time $t+1$, the 6$^{th}$ user rated the 5$^{th}$ item, the 6$^{th}$ user's taste vector $a_6$ should be updated with Eq. \ref{Eq_OnlineUserTaste}, which states that the expected taste vector would be updated by a term with respect to the newly rated item $b_5$. At the same time, the property vector of the 5$^{th}$ item should also be changed due to the 6$^{th}$ user. It can be found that the taste vector of 4$^{th}$ user is also determined by his/her rated items, including $I_5$ and $I_7$, and the updating of the 6$^{th}$ user's taste vector would cause the updating of the 4$^{th}$ user via their commonly rated item, $I_5$. At time $t+2$, it can be found that each node in the graph requires to be re-estimated when there is any newly observed rating. And the model parameters should be updated based on the newly estimated taste and property vectors accordingly. 
 
Practically, we can divide a massive dataset into small ones, and import them into Alg. \ref{Alg_OnlVarHeMF} sequentially. As stated previously, our method can provide a better solution to explain the entire massive dataset after each updating pass.  
 
\section{Empirical Variational Method}

For both bVB and oVB, a significant issue is that the hyper-parameters may be unknown empirically. This section addresses the issue and develops the empirical Variational Method~\cite{2006BishopPRML} for HeMF, which updates the hyper-parameters for each iteration. This method can further optimize the lower bound of $\mathcal{L}_{q(Y,\Theta)}(R_{\Omega}|\Xi)$ resulting to better performance. By taking the derivative of the lower bound with respect to each hyper-parameter, we derive the updating scheme as follows.

{\bf Optimizing the noise variance $\sigma^2$:} By setting the derivative $\frac{\partial \mathcal{L}(\sigma^2)}{\partial \sigma^2} = 0$, we get the optimal solution for $\sigma^2$ as follows,
\begin{equation}
\begin{split}
\sigma^2& = \frac{1}{|\Omega|}\sum_{(i,j)\in\Omega} r_{i,j}^Tr_{i,j}-2\langle b_j^T\rangle \langle a_i\rangle r_{i,j} + \mbox{Tr}\langle b_jb_j^T\rangle\langle a_ia_i^T\rangle,
\end{split}
\label{MStep_Offline}
\end{equation}
where $|\Omega|$ denotes the number of observed entries. In the online setting, the noise variance $\sigma^2$ can be updated with
\begin{equation}
\begin{split}
&{\sigma^2}^{t+1} = \frac{|\Omega|}{|\Omega\cup\Omega'|}{\sigma^2}^t\\
&+\frac{1}{|\Omega\cup\Omega'|}\sum_{(i,j)\in\Omega'} r_{i,j}^Tr_{i,j}-2\langle b_j^T\rangle \langle a_i\rangle r_{i,j} + \mbox{Tr}\langle b_jb_j^T\rangle\langle a_ia_i^T\rangle.
\end{split}
\label{MStep_Online}
\end{equation}

{\bf Optimizing the parameter over Beta distributions $\alpha$ and $\beta$:} The derivative of $\mathcal{L}(\alpha)$ with respect to $\alpha$ can be written as follows
\begin{equation}
\begin{split}
\mathcal{L}'&(\alpha)=\sum_{d=1}^{D}[ \psi(\eta_{2,d}) - \psi(\eta_{1,d}+\eta_{2,d}) ]\\
& + (\alpha - 1) D[\ln {\Gamma(\alpha)} + \psi(\alpha)]
\end{split}
\end{equation}
The new $\alpha$ can be obtained by the gradient method as:
\begin{equation}
\alpha^{new} = \alpha^{old} - \epsilon_{\alpha} \mathcal{L}'(\alpha),
\label{Eq_MStepalpha}
\end{equation}
where $\epsilon_{\alpha}$ is the learning rate used for updating $\alpha$. By altering the notions, we can get the updating formula for $\beta$ in the same way.

{\bf Optimizing the parameter over $\mu_0$:} The lower bound function with respect to $\mu_0$ can be reformulated as follows:
\begin{equation}
\begin{split}
\mathcal{L}(\mu_0) 
& \propto \mu_0^T\langle\Sigma^{-1}\rangle\langle\mu\rangle -\frac{1}{2}\mu_0^T\langle\Sigma^{-1}\rangle\mu_0.
\end{split}
\end{equation}
And the optimal solution for $\mu_0$ is
\begin{equation}
\mu_0 = \langle\mu\rangle = \frac{1}{D}\sum_{d=1}^D\mu_d.
\label{Eq_MStepMu}
\end{equation}
The optimal solution for $\nu_0$ can be achieve alternatively as $\nu_0 = \langle\nu\rangle$.

{\bf Optimization the parameter over $\lambda_0$:} $\lambda_0$ is a hyperparameter controlling the diversity of the prior distribution of $\mu_0$. By setting $\frac{\partial \mathcal{L}(\lambda_0)}{\partial \lambda_0} = 0$, we can solve the optimal solution for $\lambda_0$ with,
\begin{equation}
\begin{split}
&\lambda_0=\\
& \frac{\sum_{d=1}^{D} (\mbox{tr}\langle\mu_d\mu_d^T\rangle\langle\Sigma_d^{-1}\rangle-2\mu_0^T\langle\Sigma_d^{-1}\rangle\langle\mu_d\rangle + \mu_0^T\langle\Sigma_d^{-1}\rangle\mu_0)}{L(D+K)}\\
&+ \frac{\sum_{k=1}^{K} (\mbox{tr}\langle\nu_k\nu_k^T\rangle\langle\Pi_k^{-1}\rangle-2\nu_0^T\langle\Pi_k^{-1}\rangle\langle\nu_k\rangle + \nu_0^T\langle\Pi_k^{-1}\rangle\nu_0)}{L(D+K)},
\end{split}
\label{Eq_Msteplambda}
\end{equation}
which summarizes the diversity of the model parameters $\{\mu_d\}_{d=1}^D$ and $\{\nu_k\}_{k=1}^K$.

{\bf Optimizing the parameter over $W_0$ and $\iota_0$:} Finally, we derive the updating formulas for hyperparamters controlling the inverse-wishart distribution. As shown in \cite{1979haff}, the detailed formulation of inverse-wishart distribution can be written as,
\begin{equation}
\begin{split}
&\ln i\mathcal{W}(\Sigma_d|W_0,\iota_0)=-0.5\mbox{tr}(W_0\Sigma_d^{-1})-\frac{\iota_0+L+1}{2}\ln|\Sigma_d|\\
&+\frac{\iota_0}{2}\ln|W_0|-\frac{\iota_0L}{2}\ln 2 - \ln\Gamma_L(\frac{\iota_0}{2})\\
\end{split}
\end{equation}
Thus, by taking the derivative of $\mathcal{L}$ with respect to $W_0$ and $\iota_0$, we can get the updating formulas as follows,

\begin{equation}
\begin{split}
&W_0 = {(D+K)\iota_0}(\sum_{d=1}^{D}\Sigma_d^{-1}+\sum_{k=1}^{K}\Pi_k^{-1})^{-1}\\
&\hspace{15mm}\iota_0^{new} = \iota_0^{old} - \epsilon_{\iota_0} \mathcal{L}'(\iota_0),
\end{split}
\label{Eq_MStepSigma}
\end{equation}
with
\begin{equation}
\begin{split}
&\mathcal{L}'(\iota_0)= \frac{(D+K)}{2}[\ln|W_0|-L\ln 2-\psi_L(\frac{\iota_0}{2})] \\
&-\frac{1}{2}(\sum_{d=1}^D \ln|\Sigma_d|+\sum_{k=1}^K \ln|\Pi_k|)\\
\end{split}
\end{equation}

The detailed procedures of optimizing the hyper-parameters can be found in Alg. \ref{Alg_EmpVB}. Since all of the hyper-parameters are updated by the gradient method, our empirical variational methods are theoretically guaranteed to converge.

\begin{algorithm}
 \begin{algorithmic}[1]
 	\renewcommand{\algorithmicrequire}{\textbf{Input:}}
 	\renewcommand{\algorithmicensure}{\textbf{Output:}}
 	\REQUIRE The rating $R$ observed on positions $\Omega$
 	\ENSURE  The completed matrix $\hat{R}$
 	\\ \textit{Initialisation} :
 	\WHILE{non converge}
 	\STATE The estimating procedures in Alg. \ref{Alg_VarHeMF} or Alg. \ref{Alg_OnlVarHeMF}.
 	\STATE Optimize the noise variance $\sigma^2$ by Eq. \ref{MStep_Offline} or \ref{MStep_Online}.
 	\STATE Optimize the hyperparameter of Beta distributions\\ $\alpha$ and $\beta$ by Eq. \ref{Eq_MStepalpha}
 	\STATE Optimize the hyperparameter $\mu_0,\nu_0$ by Eq. \ref{Eq_MStepMu}
 	\STATE Optimize the hyperparameter $\lambda_0$ by Eq. \ref{Eq_Msteplambda}
 	\STATE Optimize the hyperparameter $W_0$ and $\iota_0$ by Eq. \ref{Eq_MStepSigma}
 	\ENDWHILE
 	\RETURN $\hat{R} = \langle A\rangle^T\langle B\rangle$ 
 \end{algorithmic} 
 \caption{empirical Variational Bayesian (eVB) inference for HeMF}
 \label{Alg_EmpVB}
\end{algorithm}

%
%
%
%
%

\section{Experiments and Discussions}
To evaluate the performance of the proposed methods, we apply them on 4 benchmark movie rating collaborative filtering datasets, i.e., the Netflix Prize dataset\footnote{http://www.netflixprize.com/}, the EachMovie dataset\footnote{http://grouplens.org/datasets/eachmovie/}, the 1M and Latest MovieLens Datasets\footnote{http://grouplens.org/datasets/movielens/}. We take two state-of-the-art methods as the baseline approaches, SGD and M$^3$F \cite{2010MackeyM3F}, whose source codes can be downloaded from the website\footnote{https://code.google.com/p/m3f/}.  For a fair comparison, all experiments are conducted on an Intel Core i7 920 2.67GHz CPU with 12 G RAM. For SGD and M$^3$F, we use the default parameters reported in their papers, and set the same initial configurations with our methods.

A standard cross-validation technique \cite{2004BengioCV} is applied to estimate the performance of each approach. $90\%$ entries are randomly selected as the training set, while the left $10\%$ entries are treated as the unobserved dyadic for testing. To compare the online learning performance of oVB with SGD,  we further design a series of online learning experiments. After each standard cross-validation split, we further divide the training set into several smaller ones, which will then be sub-sequentially and repeatedly  imported into the online learning procedures, oVB and SGD.  As most researchers do, we evaluate the performance of all methods by RMSE of testing dataset, which is defined as:
\begin{equation}
RMSE =  \sqrt{\frac{1}{\sum_{i,j\in\Omega}1}\sum_{i,j\in\Omega}(r_{ij}-\hat{r}_{ij})^2},
\end{equation} 
where $\hat{r}_{ij}$ denotes the predicted rating value. It is obvious that a smaller RMSE value indicates better performance.

In all the experiments conducted in this paper, we apply an empirical variational Bayesian (eVB) updating procedures after both batch variational Bayesian (bVB) and online variational Bayesian (oVB). During the implementation of oVB, we observe that some components introduced at early sequential pass would become useless. We thus introduce a mechanism to merge similar communities. The only required parameters for our methods are the learning rates used in the eVB procedures, which is set to be $0.001$ simply in all the experiments.
 
\subsection{1M MovieLens and EachMovie Datasets}
We first evaluate our models on the smaller datasets, 1M MovieLens and EachMovie, which can be effectively processed by bVB and M$^3$F models. The EachMovie dataset contains 2.8 million ratings in $\{1,\ldots,6\}$ distributed across 1648 movies and 74424 users. The 1M movieLens dataset has 6040 users, 3952 movies, and 1 million ratings in  $\{1,\ldots,5\}$. Following the "weak generalization" ratings prediction experiment in \cite{2004MarlinBPMF,2010MackeyM3F}, for each user in the training set, we withhold a single rating for the test set. All reported results are averaged over the same 10 random train-test cross-validate splits used in \cite{2004MarlinBPMF,2010MackeyM3F}. 

\begin{table}
	\centering
	\begin{tabular}{ lll }
		\hline
		\multicolumn{3}{ c }{\bf 1M MovieLens} \\
		\hline
		Method & Training error & Testing error \\ \hline
		M$^3$F-TIB(1,1,30)  &0.8029 &0.8491\\
		M$^3$F-TIB(2,1,40)  &0.7978 &0.8478\\
		M$^3$F-TIF(1,2,30) & 0.8015 &0.8484\\
		M$^3$F-TIF(2,2,40)  &0.8012 &0.8489\\
		HeMF-bVB &0.7723 &{\bf 0.8311}\\ 
		\hline
		\multicolumn{3}{ c }{\bf EachMovie} \\
		\hline
		Method & Training error & Testing error \\ \hline
		M$^3$F-TIB(1,2,40) &0.5924 &1.0993\\
		M$^3$F-TIB(2,1,40) &0.6145 &1.0920\\
		M$^3$F-TIF(1,1,30) &0.7863 &1.1021\\
		M$^3$F-TIF(1,2,30) & 0.7846 &1.1040\\
		HeMF-bVB & 0.8835&{\bf 1.0845}\\ 
		\hline
	\end{tabular}
	\caption{1M MoiveLens and EachMovie RMSE score for two M$^3$F models with optimal configurations, All scores are averaged across 10 standardized cross-validation splits. Parentheses indicate topic counts as well as the initial hidden dimensionality for  M$^3$F models $(K^U,K^M,D)$. Best results for each dataset have been boldened. }
	\label{tab_1meach}
\end{table}

Table \ref{tab_1meach} reports the predictive performance of the proposed Heterogeneous Matrix Factorization (HeMF) Model inferred by bVB method. To achieve the best performance of competitive methods, we implement M$^3$F \cite{2010MackeyM3F} with a variety of factor dimensionality and topic counts. The 4 optimal configurations of M$^3$F has been selected for our comparisons. All results are reported as the normalized mean average error.

\begin{figure}[h!]
	\centering
	\includegraphics[width=0.9\textwidth]{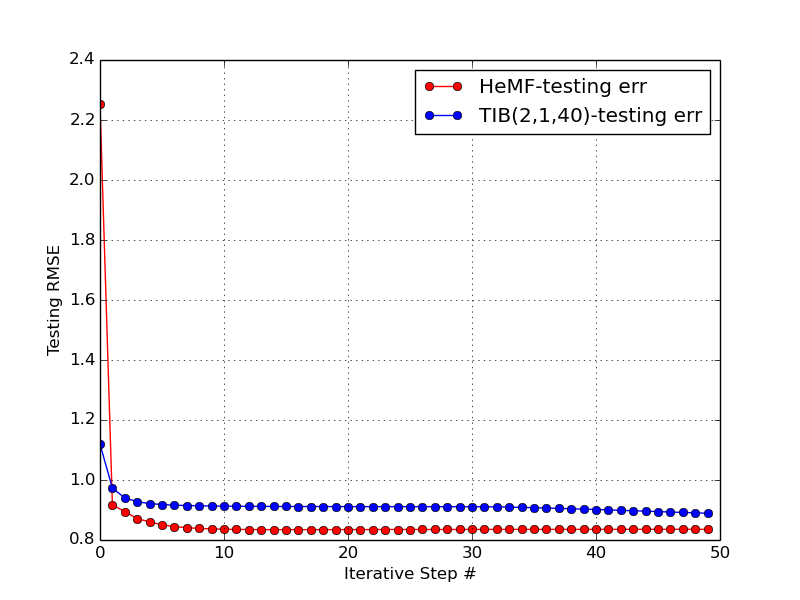}
	\caption{The testing RMSE of M$^3$F and HeMF on 1M MovieLens dataset.}
	\label{fig_1mHeMF}
\end{figure}

It can be observed that our proposed HeMF model inferred by bVB systematically outperforms the current state-of-the-art models, M$^3$F models, on both the 1M Movielens and EachMovie datasets. Note that, for EachMovie datasets, it can be found that all M$^3$F methods achieve better training RMSE results, but with worse testing performance, which indicates that M$^3$F suffers from the overfitting problem while the the proposed HeMF model well avoid it. As shown in Fig. \ref{fig_1mHeMF}, even starting with a worse initial configuration, the proposed HeMF model can converges to a better optimal solution using less iterative steps.

Since the initial hidden dimensionality of each component in HeMF is set as 20, the results indicate that the assumption of mixed membership offers greater predictive power than simply varying the dimensionality and topic counts in M$^3$F models. Moreover, with the help of Dirichlet Process, the component of groups and the corresponding dimensionality can be determined automatically during the leaning process.

We further compare our online variational Bayesian (oVB) inference approach with the most recent sequential learning of Matrix Factorization, SGD. Starting with the same initial configuration, both SGD and oVB are fed with the same sequential observations, which contains 30 samples each pass in this experiment. Fig. \ref{fig_1mSeq} illustrates the predictive performance of our sequential variational approach and the stochastic gradient descent method. It can be observed that the solution given by SGD always vibrates around an unsatisfied point, while the testing RMSE given by oVB decreases continuously and finally converges to a local optimal solution. 

To investigate the performance of the proposed oVB under different conditions, we further test it on sequential datasets with different size. As confirmed by Fig. \ref{fig_1mSeqComp}, oVB converges much faster and gets better performance when fed with more samples in each sequential pass. 

\begin{figure}[h!]
	\centering
	\includegraphics[width=0.9\textwidth]{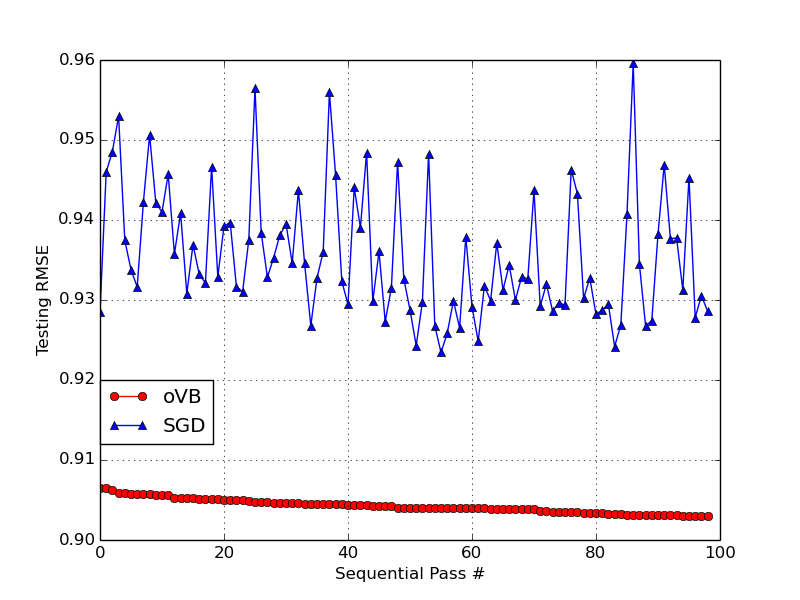}
	\caption{The testing RMSE of oVB and SGD on 1M MovieLens dataset.}
	\label{fig_1mSeq}
\end{figure}

\begin{figure}[h!]
	\centering
	\includegraphics[width=0.9\textwidth]{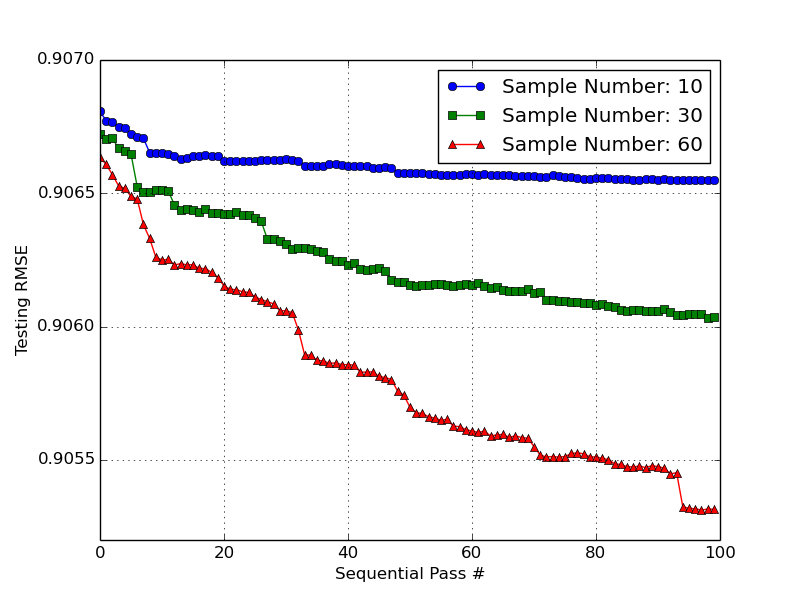}
	\caption{The testing RMSE of oVB, on 1M MovieLens dataset, when fed with samples of different size in each sequential pass.}
	\label{fig_1mSeqComp}
\end{figure}

\subsection{Latest MovieLens Dataset} 
We further test our methods on a larger dataset, the Latest MovieLens Dataset, which contains only 21,622,187 ratings distributed across 234,934 users and 30,106 movies. Thus, only $0.3\%$ entries are observed in this matrix, which is much more challenging for the DDP task. Since M$^3$F-TIB model with the configuration $(2,1,40)$ achieves the best performance among the M$^3$F family,  we set it as our baseline. 

The performance of each iterative step conducted in  M$^3$F-TIB  and HeMF has been illustrated in Fig. \ref{fig_ML_LATEST_Offline}. It can be observed that our method has much better performance with respect to RMSE value when starting with the same initial configuration. With the size of each sequential dataset as 60, the performance of oVB can be found in Fig. \ref{fig_1mLatestSeq}, which shows that oVB can achieve better results when given much richer data. When fed with more sequential data, the performance of traditional SGD becomes worse, which indicates that the SGD suffers from serious overfitting problem, which is the same as results on smaller datasets.

\begin{figure}[h!]
	\centering
	\includegraphics[width=0.9\textwidth]{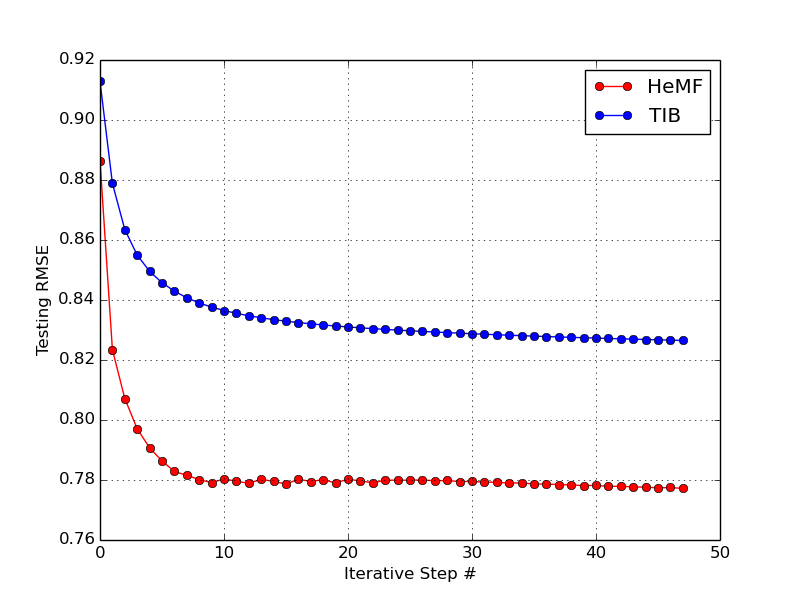}
	\caption{The testing RMSE of HeMF and TIB on Latest MovieLens dataset.}
	\label{fig_ML_LATEST_Offline}
\end{figure}

\begin{figure}[h!]
	\centering
	\includegraphics[width=0.9\textwidth]{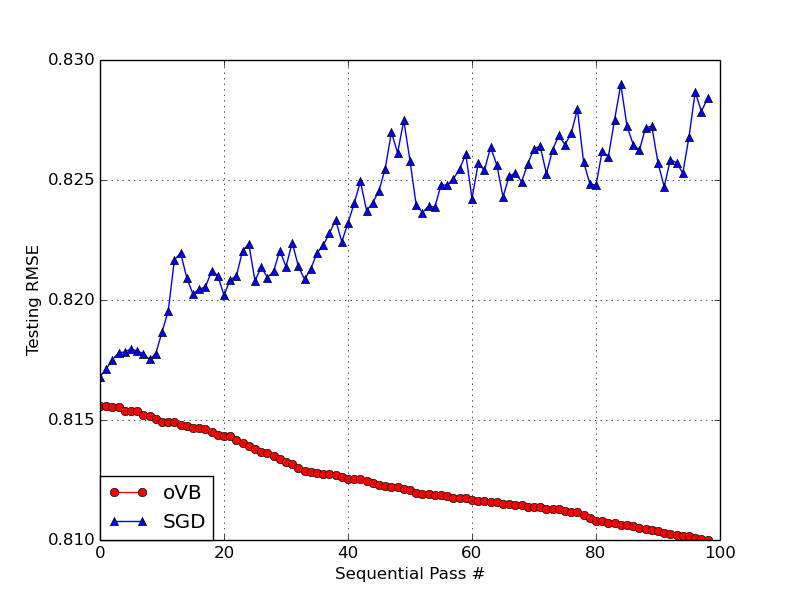}
	\caption{The testing RMSE of oVB and SGD on Latest MovieLens dataset, when feed with 60 samples  in each sequential pass.}
	\label{fig_1mLatestSeq}
\end{figure}

\subsection{Netflix Dataset}
We further implement all methods on the famous Netflix Prize dataset, which contains 100 million ratings in $\{1,\ldots,5\}$ distributed across 17,770 movies and 480,189 users. Because the public evaluation on the unobserved ratings for dyad Netflix Prize dataset is no longer available, we apply a standard cross-validation split on the observed ratings of Netflix Prize dataset, with the same setting on 1M MovieLens and EachMovie Datasets. This section sets M$^3$F-TIB model with the configuration $(2,1,200)$ as our baseline.

Fig. \ref{fig_Netflix_Offline} and \ref{fig_NetflixSeq} show the experimental results. It can be observed that HeMF can achieve better results, while the performance of TIB model get even worse with more iterative steps, which shows that the HeMF has better predictive power.

Fig. \ref{fig_Netflix_Offline} illustrates the online performance of oVB and SGD. It can be observed that the proposed SVA achieves better testing RMSE, while SGD gives an unstable estimation, and even get worse when fed with more samples.

\begin{figure}[h!]
	\centering
	\includegraphics[width=0.9\textwidth]{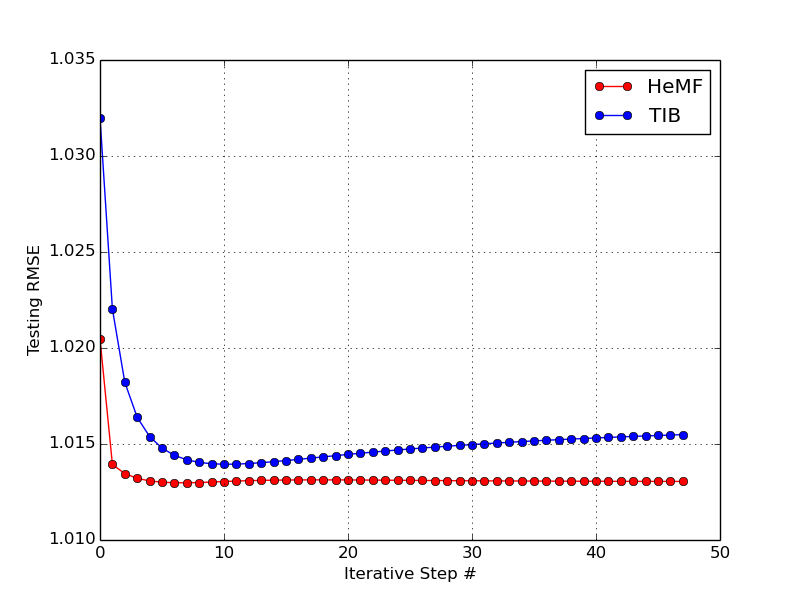}
	\caption{The testing RMSE of HeMF and TIB on Netflix dataset.}
	\label{fig_Netflix_Offline}
\end{figure}

\begin{figure}[h!]
	\centering
	\includegraphics[width=0.9\textwidth]{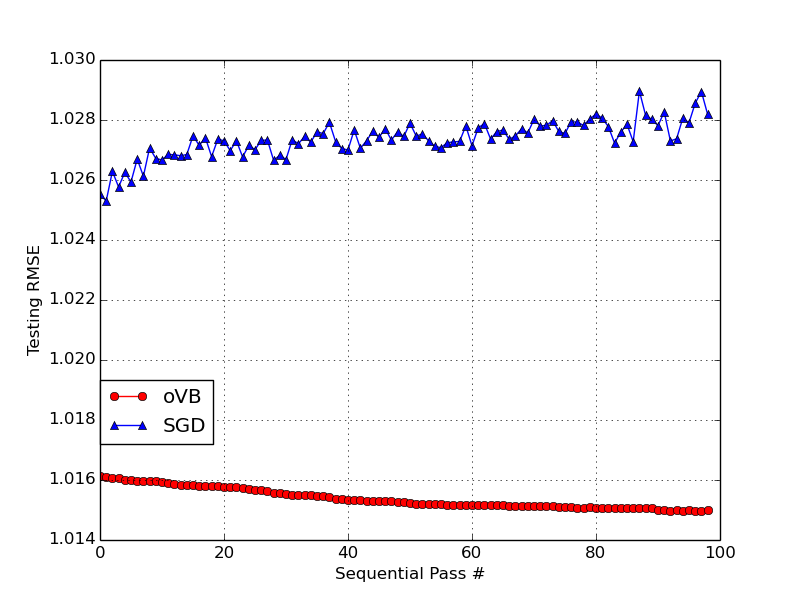}
	\caption{The testing RMSE of oVB and SGD on Netflix dataset, when feed with 60 samples in each sequential pass.}
	\label{fig_NetflixSeq}
\end{figure}
\section{Conclusions}
In this work, we developed a novel Bayesian dyadic data prediction model which integrates the complementary approaches of discrete mixed membership modeling and continuous latent factor modeling, which successfully accounts the heterogeneous property of users and their interaction in the practical recommendation systems. Two variational methods were derived to solve the proposed DDP task, one is the batch variational approximation, another is the online variational method for large-scale DDP. The performance of our methods were evaluated on real datasets including EachMovie, MovieLens and Netflix Prize. On each dataset, we found that our HeMF model achieved superior performance compared with the M$^3$F models. Specially, our online learning method oVB shows significant improvement compared with SGD, not only on the estimation accuracy but also on its robustness. It can be indicated that the proposed HeMF model with the inference methods can be a better candidate to cope with real-world DDP tasks, especially the application of recommendation systems. And the proposed online learning algorithm can well cope with the large-scale DDP problem in real-world applications.


\bibliographystyle{IEEEtran}
\bibliography{egbib}

\end{document}